\documentclass[journal]{IEEEtran}
\usepackage{amsthm}
\usepackage{pifont}
\usepackage{amsfonts}
\usepackage{amsmath,amssymb}
\usepackage{psfrag}
\usepackage{epsfig}
\usepackage{cite}
\usepackage{graphics}
\usepackage{color}
\usepackage{subfigure}
\usepackage[justification=centering]{caption}
\usepackage{multirow}
\usepackage{threeparttable}
\usepackage{booktabs}
\usepackage{cases}
\usepackage{algorithm}
\usepackage{algorithmic}
\usepackage{bm}
\usepackage{algorithmic}
\usepackage{algorithm}
\usepackage{setspace}
\usepackage{color,xcolor}
\usepackage{psfrag}
\usepackage{algorithmic,algorithm}
\usepackage{multirow}
\usepackage{threeparttable}
\usepackage{booktabs}
\usepackage{framed}
\usepackage{booktabs}
\usepackage{hyperref}
\usepackage{cleveref}

\newtheorem{theorem}{Theorem}
\newtheorem{assumption}{Assumption}
\newtheorem{lemma}{Lemma}
\newtheorem{remark}{Remark}
\newtheorem{definition}{Definition}
\begin{document}
\renewcommand{\qedsymbol}{}
\title{Activated Gradients for Deep Neural Networks}

\author{Mei Liu, Liangming Chen, Xiaohao Du, Long Jin, \IEEEmembership{Senior Member, IEEE}, and Mingsheng Shang

    \thanks{This work is supported in part by the National Natural Science Foundation of China (No. 61703189).
    }
    \thanks{M. Liu, L. Chen, X. Du, and L. Jin are with the School of Information Science and Engineering, Lanzhou University, Lanzhou, China (e-mails: jinlongsysu@foxmail.com; longjin@ieee.org). }
    \thanks{M. Shang is with the University of the Chinese Academy of Sciences, Beijing 100049, China. }
    \thanks{Kindly note that M. Liu, L. Chen, and X. Du are jointly of the first authorship. }
}
\markboth{}
{Shell \MakeLowercase{\textit{et al.}}: Bare Demo of IEEEtran.cls
    for Journals} \maketitle

\begin{abstract}
    Deep neural networks often suffer from poor performance or even training failure due to the ill-conditioned problem, the vanishing/exploding gradient problem, and the saddle point problem. In this paper, a novel method by acting the gradient activation function (GAF) on the gradient is proposed to handle these challenges. Intuitively, the GAF enlarges the tiny gradients and restricts the large gradient. Theoretically, this paper gives conditions that the GAF needs to meet, and on this basis, proves that the GAF alleviates the problems mentioned above. In addition, this paper proves that the convergence rate of SGD with the GAF is faster than that without the GAF under some assumptions. Furthermore, experiments on CIFAR, ImageNet, and PASCAL visual object classes confirm the GAF's effectiveness. The experimental results also demonstrate that the proposed method is able to be adopted in various deep neural networks to improve their performance. The source code is publicly available at \url{https://github.com/LongJin-lab/Activated-Gradients-for-Deep-Neural-Networks}.
\end{abstract}


\section{Introduction}\label{sec.intro}

\IEEEPARstart{A}{rtificial} neural networks are applied to a lot of fields, such as machine translation \cite{machineTranslate}, speech recognition \cite{speech}, object detection \cite{objectDetect}, robotics \cite{ZX_2021, PengForce2020, LJ_RNN1}, intelligent control \cite{HuangMotor2020}, etc. In neural networks, the gradient-descent-based algorithm is one of the most widely used optimization algorithms \cite{gradient_descent}. In order to accelerate the training and promote generalization, some batch-based gradient descent methods are developed in large-scale neural network training based on the vanilla gradient descent method, namely, stochastic gradient descent (SGD), batch gradient descent, and mini-batch gradient descent \cite{SGD, MBGD}, respectively. Several variants are constructed for the batch-based gradient descent methods, such as SGD with momentum (SGDM) \cite{SGDM}, Nesterov accelerated gradients \cite{SGDN}. These gradient-based methods allow a model to generalize well with relatively low computational overhead. Even though these optimization algorithms can achieve satisfying results, the backpropagation (BP) \cite{bp} algorithm may suffer from obstacles in training. Examples include the ill-conditioned problem \cite{illconditionedproblem}, the vanishing gradient problem \cite{vanishingproblem}, the exploding gradient problem \cite{explodinggproblem}, and the saddle point problem \cite{saddleproblem}.

An ill-conditioned problem is the problem with a high condition number \cite{illconditionedproblem}. The Hessian matrix of an ill-conditioned problem has both relatively large and relatively small eigenvalues \cite{illconditionedHessian}. From the perspective of the loss landscape, there are some directions that make the loss surface wide and flat, with some other narrow and sharp directions. For an ill-conditioned problem, the gradient descent optimizer often vibrates in narrow/sharp directions but converges slowly in the wide/flat directions. The ill-conditioned problem makes the tuning of the learning rate sensitive and prevents efficient training.

During deep learning training, the gradient flows through layers with the chain derivation rule in the BP algorithm. If the initial gradient value is less than $1$ in a deep network, the gradient is prone to approach zero during the calculation process, which contributes to a failure of training on some layers: This is the vanishing gradient problem. The exploding gradient problem is the opposite of the vanishing gradient one. If the gradient's initial value is greater than $1$ and the network has a large number of layers, the final gradient can be a very large number, which might result in divergence. The vanishing and exploding gradient problems are partly attributed to improper parameter initialization and inappropriate selection of the activation function, and mainly owing to the inherent deficiency of chain derivation rule in the BP algorithm \cite{HuangOptimalPress}.

In addition to the vanishing gradient problem and the exploding gradient problem, the saddle point problem is also difficult to solve. The saddle point's mathematical meaning is that the Hessian matrix \cite{L.Wei_01} of the loss function where the first-order derivative is zero (the stationary point) is an indefinite matrix. A saddle point has a second-order derivative of less than $0$ in at least one dimension, and this endows the optimizer's potential to continue the minimization. Assume that the  Hessian matrix of the loss is denoted as $\mathcal{H}$. The eigenvalues of $\mathcal{H}$ are denoted as $\lambda_1, \lambda_2, ..., \lambda_{\check{n}}$, where the symbol ${\check{n}}$ refers to the total dimension of the parameter. Assume that the probability that an eigenvalue is greater than $0$ is $p_t$. The probability $p(\mathcal{H})$ that $\mathcal{H}$ is a positive definite matrix whose eigenvalues are all greater than $0$ is
$
    p(\mathcal{H}) = p(\lambda_1) \cdot p(\lambda_2) \cdot ... \cdot p(\lambda_{\check{n}}) = (p_t)^{\check{n}}.
$
Since $p_t<1$, a large enough ${\check{n}}$ (generally, the deep learning model contains a large number of parameters) leads to $p(\mathcal{H}) \rightarrow 0$, i.e., there is a high possibility for the initial point being a saddle point. Besides, another obstacle for deep learning training is the plateau, which are flat regions with a nearly zero gradient and slows down neural network learning \cite{Saddle}. The plateau around the saddle point makes the SGD being stuck near that point for many iterations since the gradient is close to zero \cite{Saddle}.

The ill-conditioned problem, the vanishing gradient problem, the exploding gradient problem, and the saddle point problem are not completely resolved according to the loss landscape visualization and other works \cite{li2018visualizing, keskar2019large, NIPS2017_7176, wen2018Smoothout}. A method operating on gradients called the gradient activation function (GAF) is proposed to tackle the ill-conditioned problem, the vanishing gradient problem, the exploding gradient problem, and the saddle point problem. Detailed proofs are given in Section \ref{sec.2}. Besides, based on theoretical and empirical analyses, hyperparameter determination suggestions are provided, which allows the GAF to be applied in practice. Moreover, deep neural networks equipped with the GAF are evaluated on ImageNet, CIFAR-$10$, CIFAR-$100$, and PASCAL visual object classes (VOC) datasets. The main contributions of this paper are as follows:

\begin{itemize}
    \item This paper proposes the GAF, which tackles the ill-conditioned problem, the vanishing gradient problem, the exploding gradient problem, and the saddle point problem in one shot by acting a certain activation function on the gradient.

    \item Theoretically, this paper proves the following results. (a) The GAF decreases the condition number of the original optimization problem under some conditions. (b) The GAF accelerates the convergence of the SGD under some assumptions. (c) The arctan-type GAF avoids the vanishing/exploding gradient problems to some extent. (d) The arctan-type GAF escapes the saddle point region faster than GD.

    \item Extensive experiments are conducted to evaluate the practical performance of the GAF, and significant improvements are observed. Experiments on ImageNet, CIFAR-$100$, CIFAR-$10$, and PASCAL VOC datasets consistently demonstrate the effectiveness of the proposed GAF.
\end{itemize}

\section{Related Work}\label{sec.2}

Optimization is one of the fundamental issues in neural network training. A large number of solutions are constructed to solve problems in optimization. In this section, we briefly review works related to the proposed GAF solution.

\textbf{Ill-conditioned problems}: The ill-conditioned problem brings instability in the network's training. Newton's method is an optimization method with second-order derivative involved \cite{gutierrez2019acceleration}. Theoretically, under certain conditions, Newton's method solves the ill-conditioned problem effectively. However, Newton's method requires the calculation of the inverse of the Hessian matrix which contains $\check{n}^2$ elements with $\check{n}$ denoting the number of parameters in a neural network. As a result, the computational overhead of Newton's method is unacceptable. Adam is an adaptive optimization method that suppresses the oscillations toward sharp directions and accelerates the descent toward flat directions \cite{chenclosing}. Generally, Adam converges faster than SGDM, but with poor generalization ability \cite{chenclosing}.

\textbf{Vanishing gradient problems}: Several methods are studied to mitigate effects of the vanishing gradient problem, such as replacing activation functions from $\rm{sigmoid}$ and $\rm{tanh}$ to rectified linear unit (ReLU) \cite{RELU}, batch normalization (BN) \cite{BN}, and residual structures \cite{ResNet, hu2018squeeze}. Before ReLU is designed, $\rm{sigmoid}$ and $\rm{tanh}$ are mainly used as activation functions in neural networks \cite{RELU, YangNeuralPress}. However, the $\rm{sigmoid}$ and $\rm{tanh}$ activation functions' derivatives are less than $1$ (especially for somewhere far away from the origin), which easily cause the gradient to tend to $0$ after applying the chain rule. ReLU and its variants (such as Gaussian error linear units \cite{gelu}, Swish \cite{swish}, exponential linear units \cite{elu}, and scaled exponential linear units \cite{selu}) are employed to mitigate the vanishing gradient problem in deep neural network models. The ReLU has a gradient of $1$ when the input is greater than $0$, preventing the gradient from vanishing or exploding in this case. Nevertheless, ReLU can cause the gradient to be $0$ when the input is less than $0$. The BN normalizes the feature map to a certain distribution to ensure stability \cite{BN}. The BN brings remarkable improvements for deep learning models and is widely used in popular backbones. BN is essentially a solution to the vanishing gradient problem in the backpropagation process. The residual structure deals with the vanishing gradient problem from a different perspective. The output of the residual block is the sum of the activated feature map and the identity mapping of the input. Deep neural networks with residual structures maintain the information through layers and thus tend to avoid the vanishing gradient problem to some extent \cite{res_v}. However, the residual structure cannot completely solve the problem of vanishing gradient, which is supported by some visualization studies \cite{li2018visualizing}, and ResNets fail to enable arbitrarily deep networks to be effectively trained.

\textbf{Exploding gradient problems}: In addition to methods for solving the vanishing gradient problem, some other methods are exploited to tackle the exploding gradient problem, such as gradient clipping \cite{GC} and weight regularization \cite{WR}. There are two types of gradient clipping: the value clipping method is to clip the gradient that exceeds a preset threshold, and the norm clipping one adjusts the gradient according to its norm \cite{GC}.

However, gradient clipping cannot promote solving vanishing gradient and saddle point problems. Another method for solving the exploding gradient problem is to use parameter regularization, more commonly known as L$1$ regularization and L$2$ regularization. These regularization methods add a norm term to the loss function to softly constrain the parameter range. If the exploding gradient occurs (i.e., the norm of the parameter becomes very large), the regularization term can ``pull back'' the weight to a relatively flat region (i.e., the region which is close to zero), thus limit the occurrence of exploding gradients to some extent \cite{WR}. Nevertheless, the regularization term still remains unresolved issues on efficiency and stability.

\textbf{Saddle point problems}: Most existing works for handling saddle points revolve around injecting noises or introducing an adaptive learning rate \cite{saddleproblem, HW2020Ada, RG}. Feasible methods include adding noise to the correct gradient \cite{RG}, randomly initializing the optimization starting point \cite{RI}, using an optimizer with adaptive learning rates \cite{adam} to escape from the saddle points. Specifically, adding the Gaussian noise to the gradient helps the gradient avoiding saddle points \cite{RG}. There exists a general phenomenon that the optimizer of a neural network model is trapped within a neighborhood of the saddle point during the early stage. When noise is injected into the gradient, there are no such initial conditions for every iteration that make the optimizer converge into saddles.

\textbf{Operations on gradients}
Gradient normalization (GN) automatically normalizes the gradient, and performance improvement is observed in deep multitask networks \cite{GN2018}. Different from the GN that operates the overall gradient distribution, the proposed GAF modifies the value of each element of the gradient to promote optimization. Gradient centralization (GC) modifies the gradient vector so that its mean value is zero. GC performs well in many tasks but lacks a theoretical guarantee. Researchers may give more theoretical evidences of GC's effectiveness in the future, but in \cite{GC2020} itself, GC changes the sign of the gradient element, which is not easy to follow. Similar to GC, the proposed GAF operates directly on gradients. The difference is that the GAF acts specific function to gradients rather than centralizes them, and we provide extensive theoretical guarantees of convergence and on the ill-conditioned problem, the vanishing/exploding gradient problems, and the saddle point problem.

In summary,  the ill-conditioned problem, the vanishing gradient problem, the exploding gradient problem, and the saddle point problem have not been completely resolved yet. This paper proposes the GAF to alleviate these problems in one shot with theoretical and empirical evidences.

\section{Theoretical Analysis}\label{sec.3}

In this section, the formal description of the GAF is given, and theoretical analyses are provided.

\subsection{Description of the GAF}\label{sec.3.1}

The current solutions to the ill-conditioned problem, the vanishing gradient problem, the exploding gradient problem, and the saddle point problem are not perfect \cite{li2018visualizing, keskar2019large, NIPS2017_7176, wen2018Smoothout}. Since these problems are directly related to the gradient, it would be helpful if there is a strategy to control the gradient to make the training process more efficient and stable. Therefore, the GAF, which acts on the gradient, is designed to be embedded in an optimizer. To lay a basis for further discussions, a definition on the GAF is given.
\begin{definition}\label{DefGAF}
    For a function $\acute{{g}}: \mathbb{R}\rightarrow \mathbb{R}$, $\acute{{g}} = \acute{{g}}({g})$ is a GAF if the following conditions are met:
    \begin{itemize}
        \item $\acute{{g}}({g})$ is second-order differentiable and monotonic increasing;
        \item $\acute{{g}}({g})$ is an odd function;
        \item There exists a number $\varepsilon  \geq 0$ that makes $\forall g \geq \varepsilon$, $\acute{{g}}({g}) \leq g$.
        \item $g \cdot \acute{g}''(g)<0$, where the superscript $''$ denotes the second-order derivative.
    \end{itemize}
\end{definition}
Note that when the GAF is used for mapping a vector or a matrix, the same notation is kept and the mapping is element-wise, i.e., for $\acute{\boldsymbol{g}}: \mathbb{R}^{\check{n}} \rightarrow \mathbb{R}^{\check{n}}$, $\acute{\boldsymbol{g}}(\boldsymbol{g})$ is used to denote the GAF on a gradient vector $\boldsymbol{g}$, where ${\check{n}}$ is the total dimension of the parameter. Besides, the following notation is used to represent the $n$th component of the gradient: $g_n(\boldsymbol{w}) = g_n = ({\partial \mathcal{L}} / {\partial \boldsymbol{w}})_n$ with loss $\mathcal{L}$ and the parameter of the involved neural network $\boldsymbol{w}$.
Then, the following GAFs are given as examples:
\begin{itemize}
    \item Arctan-type GAF: $\acute{g}(g) = \alpha \arctan(\beta g)$;
    \item Tanh-type GAF: $\acute{g}(g) = \alpha \tanh(\beta g)$;
    \item Log-type GAF: $\acute{g}(g) = \alpha (\ln(\text{ReLU}(\beta g)+1)- \ln(\text{ReLU}(-\beta g)+1))$.
\end{itemize}
In these GAFs, $\alpha$ and $\beta$ are factors that control the shape of the GAF. For example, $\alpha$ in arctan-type GAF primarily controls the range of the gradient, and once $\alpha$ is fixed, $\beta$ mainly affects the slope of the curve in the region that the gradient close to $0$.

\subsection{Ill-conditioned Problem and Convergence Analysis}\label{sec.3.1a}

In this part, analyses on the ill-conditioned situation and convergence are given as follows.

\begin{definition}\label{DF3}
    For a continuously differentiable loss function $\mathcal{L}(\boldsymbol{w})$, its gradient $\boldsymbol{g}$ is $\ell$-tightly-Lipschitz continuous if the following conditions are met.
    \begin{itemize}
        \item[(a)] If $|| \boldsymbol{w} - \tilde{\boldsymbol{w}} || \neq 0$, the following inequality holds for any $ \boldsymbol{g}(\boldsymbol{w})$ and $ \boldsymbol{g}(\tilde{\boldsymbol{w}})$\rm{:}
              \begin{equation}
                  0 < \ell = \max \frac{||\boldsymbol{g}(\boldsymbol{w}) - \boldsymbol{g}(\tilde{\boldsymbol{w}})||} {|| \boldsymbol{w} - \tilde{\boldsymbol{w}} ||}.
              \end{equation}
        \item[(b)] If $|| \boldsymbol{w} - \tilde{\boldsymbol{w}} || = 0$, $||\boldsymbol{g}(\boldsymbol{w}) - \boldsymbol{g}(\tilde{\boldsymbol{w}})||=0$.
    \end{itemize}
    Note that ${\boldsymbol{w}}$ and $\tilde{\boldsymbol{w}}$ are two arbitrary parameter vectors; $||\cdot||$ denotes the 2-norm operation.
\end{definition}

\begin{definition}\label{DF4}
    A continuously differentiable loss function $\mathcal{L}(\boldsymbol{w})$ is $c$-tightly-strongly convex if the following conditions are met.
    \begin{itemize}
        \item[(a)] If $|| \boldsymbol{w} - \tilde{\boldsymbol{w}} || \neq 0$, the following inequality holds for any $ \boldsymbol{g}(\boldsymbol{w})$ and $\boldsymbol{g}(\tilde{\boldsymbol{w}}) $\rm{:}
              \begin{equation}
                  0 < c = \min \frac{||\boldsymbol{g}(\boldsymbol{w}) - \boldsymbol{g}(\tilde{\boldsymbol{w}})||} {|| \boldsymbol{w} - \tilde{\boldsymbol{w}} ||}.
              \end{equation}
        \item[(b)] If $|| \boldsymbol{w} - \tilde{\boldsymbol{w}} || = 0$, then $||\boldsymbol{g}(\boldsymbol{w}) - \boldsymbol{g}(\tilde{\boldsymbol{w}})||=0$.
    \end{itemize}
\end{definition}

\begin{remark}
    If $|| \boldsymbol{w} - \tilde{\boldsymbol{w}} || \neq 0$, according to Definition \ref{DF3},
    \[\ell = \max \frac{||\boldsymbol{g}(\boldsymbol{w}) - \boldsymbol{g}(\tilde{\boldsymbol{w}})||} {|| \boldsymbol{w} - \tilde{\boldsymbol{w}} ||} \geq \frac{||\boldsymbol{g}(\boldsymbol{w}) - \boldsymbol{g}(\tilde{\boldsymbol{w}})||} {|| \boldsymbol{w} - \tilde{\boldsymbol{w}} ||}. \]
    Then, $\ell || \boldsymbol{w} - \tilde{\boldsymbol{w}} ||\geq ||\boldsymbol{g}(\boldsymbol{w}) - \boldsymbol{g}(\tilde{\boldsymbol{w}})||$ holds. If $|| \boldsymbol{w} - \tilde{\boldsymbol{w}} || = 0$, $\ell || \boldsymbol{w} - \tilde{\boldsymbol{w}} ||\geq ||\boldsymbol{g}(\boldsymbol{w}) - \boldsymbol{g}(\tilde{\boldsymbol{w}})||$ maintains. Thus, $\ell$-tightly-Lipschitz continuous implies that the gradient is Lipschitz continuous with Lipschitz constant $\ell$. The similar conclusion is drawn from $c$-tightly strongly convex to strongly convex. For simplification, $\boldsymbol{g}(\boldsymbol{w})$ and $\boldsymbol{g}(\tilde {\boldsymbol{w}})$ are written as $\boldsymbol{g}$ and $\tilde {\boldsymbol{g}}$, respectively.
\end{remark}

\begin{lemma}\label{LM1}
    Suppose that the GAF $\acute{\boldsymbol{g}}: \mathbb{R}^{\check{n}} \rightarrow \mathbb{R}^{\check{n}}$ is continuously differentiable with its component $\acute{g}: \mathbb{R} \rightarrow \mathbb{R}$. In addition, suppose that $\acute{g}'({g}_n)>1$ for any $ |{g}_n| \leq \epsilon_0$. Then, $||\acute{\boldsymbol{g}} (\bar{\boldsymbol{g}})-\acute{\boldsymbol{g}}(\tilde{\boldsymbol{g}})||>||\bar{\boldsymbol{g}}-\tilde{\boldsymbol{g}}||$ for $ |{g}_n| \leq \epsilon_0$, where the superscript $'$ denotes the first-order derivative; $\bar{\boldsymbol{g}}$ and $\tilde{\boldsymbol{g}}$ are two arbitrary gradient vectors; $||\bar{\boldsymbol{g}}-\tilde{\boldsymbol{g}}|| \neq 0$.
\end{lemma}

\begin{proof}
    If $\bar{g}_n > \tilde{g}_n$, by Lagrange's mean value theorem \cite{o2017advanced}, one obtains    $\acute{{g}} (\bar{g}_n)-\acute{{g}}(\tilde{{g}}_n) =  \acute{g}'(\xi_1) (\bar{g}_n-\tilde{g}_n)$,
    where $\xi_1 \in (\bar{g}_n, \tilde{g}_n)$. Since $\acute{g}'(\xi_1)>1$, it follows that
    \begin{equation}\label{435}
        \acute{{g}} (\bar{g}_n)-\acute{{g}}(\tilde{{g}}_n) > (\bar{g}_n-\tilde{g}_n).
    \end{equation}
    If $\bar{g}_n < \tilde{g}_n$, by the similar steps, we have
    \begin{equation}\label{436}
        \acute{{g}} (\bar{g}_n)-\acute{{g}}(\tilde{{g}}_n) < (\bar{g}_n-\tilde{g}_n).
    \end{equation}
    Together, equations \eqref{435} and \eqref{436} yield $(\acute{{g}} (\bar{g}_n)-\acute{{g}}(\tilde{{g}}_n))^2 > (\bar{g}_n-\tilde{g}_n)^2$.
    Summing over all dimensions gives the result:
    \begin{equation}\label{438}
        \begin{aligned}
            ||\acute{\boldsymbol{g}} (\bar{\boldsymbol{g}})-\acute{\boldsymbol{g}}(\tilde{\boldsymbol{g}})|| = & (\sum_{n=1}^{\check{n}} (\acute{{g}} (\bar{g}_n)-\acute{{g}}(\tilde{{g}}_n))^2)^{\frac{1}{2}}
            \\ &> (\sum_{n=1}^{\check{n}} (\bar{g}_n-\tilde{g}_n)^2)^{\frac{1}{2}} = ||\bar{\boldsymbol{g}}-\tilde{\boldsymbol{g}}||.
        \end{aligned}
    \end{equation}
    The proof is completed.
\end{proof}

\begin{lemma}\label{LM2}
    Suppose that the GAF $\acute{\boldsymbol{g}}: \mathbb{R}^{\check{n}} \rightarrow \mathbb{R}^{\check{n}}$ is continuously differentiable with its component $\acute{g}: \mathbb{R} \rightarrow \mathbb{R}$. In addition, suppose that $\exists \epsilon_1 \leq \epsilon_2$ which satisfies $\acute{g}'(\epsilon_1) \leq 1$ with $\acute{g}(\epsilon_2) \leq \epsilon_2$. Then, $||\acute{\boldsymbol{g}} (\bar{\boldsymbol{g}})-
        \acute{\boldsymbol{g}} (\tilde{\boldsymbol{g}})|| < ||\bar{\boldsymbol{g}}-\tilde{\boldsymbol{g}}||$ for any $|g_n| > \epsilon_2$.
\end{lemma}

\begin{proof}
    Consider two arbitrary scalars $\bar{g}_n$ and $\tilde{g}_n$ respectively satisfy $|\bar{g}_n| > \epsilon_2$ and $|\tilde{g}_n| > \epsilon_2$.
    \begin{itemize}
        \item [1)] First, we consider the case of $\bar{g}_n>\epsilon_2>0$ and $\tilde{g}_n < -\epsilon_2<0$. By Lagrange's mean value theorem \cite{o2017advanced}, $\exists \xi_2 \in (\epsilon_1, \epsilon_2)$ makes
              \begin{equation}\label{999}
                  \acute{g}'(\epsilon_1)- \acute{g}'(\epsilon_2)=\acute{g}''(\xi_2) (\epsilon_1 - \epsilon_2).
              \end{equation}
              According to the definition on the GAF (Definition \ref{DefGAF}), $\xi_2>0$ and $\xi_2 \cdot \acute{g}''(\xi_2)<0$ yield $\acute{g}''(\xi_2)<0$.
              Together with $\acute{g}'(\epsilon_1) \leq 1$ and \eqref{999}, it has
              \begin{equation}\label{439}
                  \acute{g}'(\epsilon_2) < 1.
              \end{equation}
              Using Taylor’s expansion \cite{o2017advanced} at $\bar{g}_n$ and $\epsilon_2$, we have
              \begin{equation}\label{440}
                  \acute{g}(\bar{g}_n)= \acute{g}(\epsilon_2) + \acute{g}'(\epsilon_2)(\bar{g}_n - \epsilon_2)+\frac{1}{2}(\bar{g}_n - \epsilon_2)^2 \acute{g}''(\xi_3),
              \end{equation}
              where $\xi_3 \in (\bar{g}_n, \epsilon_2)$. Under the assumption that $\acute{g}(\epsilon_2) \leq \epsilon_2$, considering equations \eqref{439} and \eqref{440} together yields
              \begin{equation}
                  \begin{aligned}
                      \acute{g}(\bar{g}_n) & \leq \epsilon_2 + \acute{g}'(\epsilon_2)(\bar{g}_n - \epsilon_2)+\frac{1}{2}(\bar{g}_n - \epsilon_2)^2 \acute{g}''(\xi_3)
                      \\ & \leq \bar{g}_n +\frac{1}{2}(\bar{g}_n - \epsilon_2)^2 \acute{g}''(\xi_3).
                  \end{aligned}
              \end{equation}
              Since $\acute{g}''(\xi_3) < 0$,
              \begin{equation}\label{441}
                  \acute{g}(\bar{g}_n) < \bar{g}_n
              \end{equation} is obtained. By following a similar analysis, we have
              \begin{equation}\label{442}
                  \acute{g}(\tilde{g}_n) > \tilde{g}_n.
              \end{equation}
              Combining \eqref{441}, \eqref{442}, and the definition on the GAF generates
              \begin{equation}\label{443}
                  ( \acute{g}(\bar{g}_n) - \acute{g}(\tilde{g}_n) )^2 < (\bar{g}_n - \tilde{g}_n)^2.
              \end{equation}
              Then,
              \begin{equation}\label{444}
                  ||\acute{\boldsymbol{g}} (\bar{\boldsymbol{g}})-\acute{\boldsymbol{g}}(\tilde{\boldsymbol{g}})|| < ||\bar{\boldsymbol{g}}-\tilde{\boldsymbol{g}}||
              \end{equation}
              follows by summing \eqref{443} through all dimensions.

        \item [2)] For the case of $\bar{g}_n> \tilde{g}_n>\epsilon_2$, using Lagrange's mean value theorem \cite{o2017advanced} for $\tilde{g}_n<\xi_5<\bar{g}_n$ with $\epsilon _1<\xi_4<\xi_5$ leads to
              \[\acute{g}'(\xi_5)- \acute{g}'(\epsilon_1) = \acute{g}''(\xi_4)(\xi_5 - \epsilon_1).\]
              Combining $\acute{g}''(\xi_4)<0$ (directly obtained from $0<\epsilon _1<\xi_4$ and the assumption that $\xi_4 \acute{g}''(\xi_4)<0$), $\xi_5 - \epsilon_1>0$, and $\acute{g}'(\epsilon_1) \leq 1$, it gives that $\acute{g}'(\xi_5) < 1$. Then, using Lagrange's mean value theorem \cite{o2017advanced} for $\tilde{g}_n<\xi_5<\bar{g}_n$ gives
              \begin{equation}\label{445}
                  \acute{g}(\bar{g}_n)- \acute{g}(\tilde{g}_n) < \bar{g}_n - \tilde{g}_n.
              \end{equation}
              Since \eqref{445} holds for all dimensions,
              \begin{equation}\label{4451}
                  ||\acute{\boldsymbol{g}} (\bar{\boldsymbol{g}})-\acute{\boldsymbol{g}}(\tilde{\boldsymbol{g}})|| < ||\bar{\boldsymbol{g}}-\tilde{\boldsymbol{g}}||
              \end{equation}
              is obtained.

        \item [3)] Similarly, for the case of $\bar{g}_n< \tilde{g}_n< - \epsilon_2$, we have $            \acute{g}(\bar{g}_n)- \acute{g}(\tilde{g}_n) > \bar{g}_n - \tilde{g}_n$
              and then,
              \begin{equation}\label{447}
                  ||\acute{\boldsymbol{g}} (\bar{\boldsymbol{g}})-\acute{\boldsymbol{g}}(\tilde{\boldsymbol{g}})|| < ||\bar{\boldsymbol{g}}-\tilde{\boldsymbol{g}}||.
              \end{equation}
    \end{itemize}
    The proof is complete.
\end{proof}

\begin{assumption}\label{AS3}
    The loss function $\mathcal{L}:
        \mathbb{R}^{\check{n}}  \rightarrow \mathbb{R}$ is continuously differentiable and satisfies the following conditions.
    \begin{itemize}
        \item[(a)] $( w_n^{\rm{L}}, g_n^{\rm{L}}) = \{ (w_n, g_n) | ~~|g_n(\boldsymbol{w})| \leq \epsilon_0 \}$, where $w_n$ and $g_n$ are the $n$th component of the parameter and the gradient, respectively.
        \item[(b)]
              Consider that the part of the loss function that meets with this condition is composed of $\check{s}$ segments as
              \[ (w_n^{\rm{L}})_1, \cdots, (w_n^{\rm{L}})_s, \cdots, (w_n^{\rm{L}})_{\check{s}} \in w_n^{\rm{L}};\] $\forall s$, the loss is continuously differentiable on $(w_n^{\rm{L}})_{s}$, where
              \[(w_n^{\rm{L}})_1 \cup \cdots \cup (w_n^{\rm{L}})_s \cdots \cup (w_n^{\rm{L}})_{\check{s}} = w_n^{\rm{L}};\]
              \[(w_n^{\rm{L}})_1 \cap \cdots \cap (w_n^{\rm{L}})_s \cdots\cap (w_n^{\rm{L}})_{\check{s}} = \varnothing.\]
        \item[(c)] $( w_n^{\rm{H}}, g_n^{\rm{H}}) = \{ (w_n, g_n) | ~~|g_n(\boldsymbol{w})| > \epsilon_2 \}$ under the high gradient circumstance.
        \item[(d)] Consider that the part of the loss function that meets with this condition is composed of $\check{t}$ segments as \[ (w_n^{\rm{H}})_1, \cdots, (w_n^{\rm{H}})_t, \cdots, (w_n^{\rm{H}})_{\check{t}} \in w_n^{\rm{H}};\]
              $\forall t$, the loss is continuously differentiable on $(w_n^{\rm{H}})_{t}$, where
              \[(w_n^{\rm{H}})_1 \cup \cdots \cup (w_n^{\rm{H}})_t \cdots \cup (w_n^{\rm{H}})_{\check{t}} = w_n^{\rm{H}};\]
              \[(w_n^{\rm{H}})_1 \cap \cdots \cap (w_n^{\rm{H}})_t \cdots \cap (w_n^{\rm{H}})_{\check{t}} = \varnothing.\]
    \end{itemize}
\end{assumption}

\begin{assumption}\label{AS4}
    The gradient $\boldsymbol{g}:
        \mathbb{R}^{\check{n}}  \rightarrow \mathbb{R}^{\check{n}}$ is $\ell$-tightly-Lipschitz continuous and satisfies the following.
    \begin{itemize}
        \item[(a)] The Lipschitz constants of $g_n^{\rm{L}}$ and $g_n^{\rm{H}}$ are \[\ell_n^{\rm{L}} = \max\{ (\ell_n^{\rm{L}})_1, \cdots, ( \ell_n^{\rm{L}})_{s}, \cdots, ( \ell_n^{\rm{L}})_{\check{s}} \}\] and \[\ell_n^{\rm{H}} = \max\{ ( \ell_n^{\rm{H}})_1, \cdots, ( \ell_n^{\rm{H}})_{t}, \cdots, ( \ell_n^{\rm{H}})_{\check{t}} \},\] respectively. In addition, $( \ell_n^{\rm{L}})_{s}$ and $( \ell_n^{\rm{H}})_{t}$ are the Lipschitz constants of $( g_n^{\rm{L}})_{s}$ and $( g_n^{\rm{H}})_{t}$, respectively.
        \item[(b)] The Lipschitz constant $\ell_n$ of the gradient $g_n$ in dimension $n$ satisfies $\ell_n = \max\{ \ell_n^{\rm{L}}, \ell_n^{\rm{H}}\}.$
    \end{itemize}
\end{assumption}

\begin{theorem}\label{TH1}
    Under Assumption \ref{AS3} and Assumption \ref{AS4}, suppose that $\ell_n^{\rm{L}} < \ell_n^{\rm{H}}$ and $(\ell_n^{\rm{L}})^{\rm{G}} < (\ell_n^{\rm{H}})^{\rm{G}}$, where the superscript $^{\rm{G}}$ means that the GAF acts on the gradient. In addition, suppose that the Lipschitz constant of the gradient $\ell = \max \{ \ell_n | n = 1, \cdots, \check{n} \}$, $(\ell)^{\rm{G}} = \max \{ (\ell_n)^{\rm{G}} | n = 1, \cdots, \check{n} \}$, and $\ell_n^{\rm{L}} < \ell_n^{\rm{H}}$. Then,
    \begin{equation}\label{EQ22}
        (\ell)^{\rm{G}}<\ell.
    \end{equation}
\end{theorem}

\begin{proof}
    By Lemma \ref{LM2} and Definition \ref{DF3}, $\forall t$, it follows that
    \begin{equation}\label{EQ23}
        \begin{aligned}
            (\ell_n^{\rm{H}})_t & =  \max \frac{||(g_n^{\rm{H}})_t  - (\tilde{g}_n^{\rm{H}})_t ||} {|| (w_n^{\rm{H}})_t - (\tilde{w}_n^{\rm{H}})_t || } \\ &
            > \max \frac{||(g_n^{\rm{H}})_t^{\rm{G}}  - (\tilde{g}_n^{\rm{H}})_t^{\rm{G}} ||} {|| (w_n^{\rm{H}})_t - (\tilde{w}_n^{\rm{H}})_t || } = (\ell_n^{\rm{H}})_t^{\rm{G}}.
        \end{aligned}
    \end{equation}
    Under Assumption \ref{AS4}(a), \eqref{EQ23} gives
    \begin{equation}\label{EQ24}
        (\ell_n^{\rm{H}})^{\rm{G}} < \ell_n^{\rm{H}}.
    \end{equation}
    Recalling Assumption \ref{AS4}(b), $\ell_n^{\rm{L}} < \ell_n^{\rm{H}}$, and $(\ell_n^{\rm{L}})^{\rm{G}} < (\ell_n^{\rm{H}})^{\rm{G}}$ gives $(\ell_n)^{\rm{G}} = \max\{ (\ell_n^{\rm{L}})^{\rm{G}}, (\ell_n^{\rm{H}})^{\rm{G}} \} = (\ell_n^{\rm{H}})^{\rm{G}}$ and $\ell_n = \max\{ \ell_n^{\rm{L}}, \ell_n^{\rm{H}} \} = \ell_n^{\rm{H}}$. Together with \eqref{EQ24}, one obtains     \begin{equation}\label{EQ24b}
        \ell_n > \ell_n^{\rm{G}}.
    \end{equation}
    Since $\ell = \max \{ \ell_n | n = 1, \cdots, \check{n} \}$, $(\ell)^{\rm{G}} = \max \{ (\ell_n)^{\rm{G}} | n = 1, \cdots, \check{n} \}$, we have $(\ell)^{\rm{G}}<\ell$. The proof is complete.
\end{proof}

\begin{assumption}\label{AS5}
    The loss function $\mathcal{L}:
        \mathbb{R}^{\check{n}}  \rightarrow \mathbb{R}$ is $c$-tightly-strongly convex and satisfies the following.
    \begin{itemize}
        \item[(a)] The strong convexity constants of $\mathcal{L} (w_n^{\rm{L}})$ and $\mathcal{L} (w_n^{\rm{H}} )$ are \[c_n^{\rm{L}} = \min\{ (c_n^{\rm{L}})_1, \cdots, ( c_n^{\rm{L}})_{s}, \cdots, ( c_n^{\rm{L}})_{\check{s}} \}\] and \[c_n^{\rm{H}} = \min\{ ( c_n^{\rm{H}})_1, \cdots, ( c_n^{\rm{H}})_{t}, \cdots, ( c_n^{\rm{H}})_{\check{t}} \},\] respectively. In addition, $( c_n^{\rm{L}})_{s}$ and $( c_n^{\rm{H}})_{t}$ are the strong convexity constants of $\mathcal{L}(( w_n^{\rm{L}})_{s})$ and $\mathcal{L}(( w_n^{\rm{H}})_{t})$, respectively.
        \item[(b)] The strong convexity constant $c_n$ of $\mathcal{L}(w_n)$ in dimension $n$ satisfies $c_n = \min\{ c_n^{\rm{L}}, c_n^{\rm{H}}\}$.
    \end{itemize}
\end{assumption}

\begin{theorem}\label{TH2}
    Under Assumptions \ref{AS3} and \ref{AS5}, suppose that $c_n^{\rm{L}} < c_n^{\rm{H}}$ and $(c_n^{\rm{L}})^{\rm{G}} < (c_n^{\rm{H}})^{\rm{G}}$ for $w_n \in [w_n^a, w_n^b]$. In addition, assume that the derivative of $\acute{g}$ with respect to $g_n$ satisfies $\acute{g}'(g_n)>1$ for any $|g_n| \leq \epsilon_0$. In addition, suppose that the strong convexity constant $c = \min \{ c_n | n = 1, \cdots, \check{n} \}$, $(c)^{\rm{G}} = \min \{ (c_n)^{\rm{G}} | n = 1, \cdots, \check{n} \}$. Then,
    \begin{equation}\label{EQ25}
        c^{\rm{G}} > c.
    \end{equation}
\end{theorem}

\begin{proof}
    By Lemma \ref{LM1} and Definition \ref{DF4}, $\forall s$, it follows that
    \begin{equation}\label{EQ26}
        \begin{aligned}
            (c_n^{\rm{L}})_s & =  \min \frac{||(g_n^{\rm{L}})_s  - (\tilde{g}_n^{\rm{L}})_s ||} {|| (w_n^{\rm{L}})_s - (\tilde{w}_n^{\rm{L}})_s || } \\ &
            < \min \frac{||(g_n^{\rm{L}})_s^{\rm{G}}  - (\tilde{g}_n^{\rm{L}})_s^{\rm{G}} ||} {|| (w_n^{\rm{L}})_s - (\tilde{w}_n^{\rm{L}})_s || } = (c_n^{\rm{L}})_s^{\rm{G}}.
        \end{aligned}
    \end{equation}
    Under Assumption \ref{AS5}(a), \eqref{EQ26} gives
    \begin{equation}\label{EQ27}
        (c_n^{\rm{L}})^{\rm{G}} < c_n^{\rm{L}}.
    \end{equation}
    Recalling Assumption \ref{AS5}(b), $c_n^{\rm{L}} < c_n^{\rm{H}}$, and $(c_n^{\rm{L}})^{\rm{G}} < (c_n^{\rm{H}})^{\rm{G}}$ gives $(c_n)^{\rm{G}} = \min\{ (c_n^{\rm{L}})^{\rm{G}}, (c_n^{\rm{H}})^{\rm{G}} \} = (c_n^{\rm{L}})^{\rm{G}}$ and $c_n = \min\{ c_n^{\rm{L}}, c_n^{\rm{H}} \} = c_n^{\rm{L}}$. Together with \eqref{EQ27}, $c = \min \{ c_n | n = 1, \cdots, \check{n} \}$, and $(c)^{\rm{G}} = \min \{ (c_n)^{\rm{G}} | n = 1, \cdots, \check{n} \}$, one obtains \eqref{EQ25}. The proof is complete.
\end{proof}

\begin{theorem}\label{THCond}
    Under Assumptions \ref{AS3}, \ref{AS4}, and \ref{AS5}, suppose that $c_n^{\rm{L}} < c_n^{\rm{H}}$, $(c_n^{\rm{L}})^{\rm{G}} < (c_n^{\rm{H}})^{\rm{G}}$, $\ell_n^{\rm{L}} < \ell_n^{\rm{H}}$, and $(\ell_n^{\rm{L}})^{\rm{G}} < (\ell_n^{\rm{H}})^{\rm{G}}$ for $w_n \in [w_n^a, w_n^b]$. In addition, assume that the derivative of $\acute{g}$ with respect to $g_n$ satisfies $\acute{g}'(g_n)>1$ for any $|g_n| \leq \epsilon_0$. Then, $\zeta^{\rm G}<\zeta$, where $\zeta^{\rm G}$ is the condition number of $\mathcal{L}$ with the GAF and $\zeta$ is the condition number without the GAF.
\end{theorem}

\begin{proof}
    The condition number is defined as $\zeta = {\ell}/{c}$ in literatures \cite{RefCD1, RefCD2}. Recalling \eqref{EQ22} and \eqref{EQ25}, we obtain $\zeta^{\rm G} = {(\ell)^{\rm G}}/{(c)^{\rm G}} < {\ell}/{c} = \zeta$. The proof is complete.
\end{proof}

\begin{assumption}\label{ASBott}
    Assume that stochastic gradient descent (SGD) (Algorithm 1 in \cite{LPOpt}) is taken as the optimizer and the following conditions are met \cite{LPOpt}.
    \begin{itemize}
        \item[(a)] The lower bound of the loss $\mathcal{L}$ is a scalar $\mathcal{L}_{\rm{bound}}$.
        \item[(b)] $\forall k \in \mathbb{N}$, $\exists \mu_G$ and $\mu$ satisfy $\mu_G \geq \mu > 0$ and make
            \[
            \nabla \mathcal{L}(\boldsymbol{w}_{k})^{\top} \mathbb{E}_{\xi_{k}}[\boldsymbol{g}(\boldsymbol{w}_{k}, \xi_{k})]  \geq \mu||\nabla \mathcal{L}(\boldsymbol{w}_{k})||^{2} \] and
            \[
            ||\mathbb{E}_{\xi_{k}}[\boldsymbol{g}(\boldsymbol{w}_{k}, \xi_{k})]||  \leq \mu_{G}||\nabla \mathcal{L}(\boldsymbol{w}_{k})||,
            \]
            where $\mathbb{E} [\cdot]$ denotes the operation of taking the expectation; $\xi_{k}$ is a random seed in iteration $k$ that takes a set of samples from the whole samples; $\boldsymbol{g}(\boldsymbol{w}_{k}, \xi_{k})$ is an unbiased estimate of $\mathcal{L}(\boldsymbol{w}_{k})$ on sampled data.
            \item[(c)] $\forall k \in \mathbb{N}$,  $\exists M \geq 0$ and $M_{V} \geq 0$ make
            \[
            \begin{aligned}
                & \mathbb{E}_{\xi_{k}} [||\boldsymbol{g}(\boldsymbol{w}_{k}, \xi_{k})||^{2}]-||\mathbb{E}_{\xi_{k}}[\boldsymbol{g}( \boldsymbol{w}_{k}, \xi_{k} )]||^{2} \\ & \leq M+M_{V}|| \nabla \mathcal{L}(\boldsymbol{w}_{k})||^{2}.
            \end{aligned}
            \]
    \end{itemize}
\end{assumption}

\begin{theorem}\label{THConver}
    Consider a stochastic convex optimization problem. Under Assumptions \ref{AS3}, \ref{AS4}, \ref{AS5}, and \ref{ASBott}, suppose that $c_n^{\rm{L}} < c_n^{\rm{H}}$, $(c_n^{\rm{L}})^{\rm{G}} < (c_n^{\rm{H}})^{\rm{G}}$, $\ell_n^{\rm{L}} < \ell_n^{\rm{H}}$, and $(\ell_n^{\rm{L}})^{\rm{G}} < (\ell_n^{\rm{H}})^{\rm{G}}$ for $w_n \in [w_n^a, w_n^b]$. In addition, assume that the derivative of $\acute{g}$ with respect to $g_n$ satisfies $\acute{g}'(g_n)>1$ for any $|g_n| \leq \epsilon_0$. Suppose that the SGD \cite{LPOpt} is the optimizer and the learning rate $\eta$ equals to a constant $\mu/ (\ell M_G)$ through all iteration, where $M_G:=M_V+\mu^2_G$. Denote the expectation of the gap between the minimum $\mathcal{L}_*$ and $\mathcal{L}(\boldsymbol{w}_{k} )$ as $\mathbb{E}[\mathcal{L}(\boldsymbol{w}_{k} )-\mathcal{L}_{*}]$, Then, $\mathbb{E}[\mathcal{L}(\boldsymbol{w}_{k} )-\mathcal{L}_{*}]$ converges faster with the GAF than without the GAF for $k = 2, 3, \cdots$.
\end{theorem}

\begin{proof}
    Following similar proof steps of Theorem 4.6 in \cite{LPOpt}, one obtains
    \begin{equation}
        \mathbb{E}[\mathcal{L}(\boldsymbol{w}_{k+1} )-\mathcal{L}_{*}]-\frac{\eta \ell M}{2 c \mu} \leq (1- \eta c \mu)(\mathbb{E}[\mathcal{L}(\boldsymbol{w}_{k})-\mathcal{L}_{*}]-\frac{\eta \ell M}{2 c \mu}).
    \end{equation}
    Recalling $\eta = \mu/ (\ell M_G)$ and rearranging yields
    \begin{equation}\label{EQ30}
        \begin{aligned}
             & \mathbb{E}[\mathcal{L}(\boldsymbol{w}_{k+1} )-\mathcal{L}_{*}] \\ & \leq (1-\frac{c \mu^2}{\ell M_G})(\mathbb{E}[\mathcal{L}(\boldsymbol{w}_{k})-\mathcal{L}_{*}]-\frac{M}{2 c M_G}) +\frac{M}{2 c M_G}.
        \end{aligned}
    \end{equation}
    By using \eqref{EQ30} for $k, k-1, k-2, \cdots$, we have
    \begin{equation}\label{EQ31}
        \begin{aligned}
             & \mathbb{E}[\mathcal{L}(\boldsymbol{w}_{k} )-\mathcal{L}_{*}] \\ &
            \leq (1-\frac{c \mu^2}{\ell M_G})^{k-1}(\mathbb{E}[\mathcal{L}(\boldsymbol{w}_{1})-\mathcal{L}_{*}]-\frac{M}{2 c M_G}) +\frac{M}{2 c M_G}.
        \end{aligned}
    \end{equation}
    Following the same procedures, a conclusion is drawn that after using the GAF as
    \begin{equation}\label{EQ563}
        \begin{aligned}
             & \mathbb{E}[\mathcal{L}(\boldsymbol{w}_{k} )-\mathcal{L}_{*}]                                                                                        \\ &
            \leq (1-\frac{(c)^{\rm G} \mu^2}{(\ell)^{\rm G} M_G})^{k-1}(\mathbb{E}[\mathcal{L}(\boldsymbol{w}_{1})-\mathcal{L}_{*}]-\frac{M}{2 (c)^{\rm G} M_G}) + \\ &\frac{M}{2 (c)^{\rm G} M_G}.
        \end{aligned}
    \end{equation}

    Since $(\ell)^{\rm G}<\ell$, $(c)^{\rm G}>c$, $(\ell)^{\rm G}>(c)^{\rm G}$, and $\mu^2 \leq M_G$ from \cite{LPOpt}, then it follows that \[0 \leq 1-\frac{(c)^{\rm G} \mu^2}{(\ell)^{\rm G} M_G} < 1-\frac{c \mu^2}{\ell M_G} < 1.\]
    Thus, $\mathbb{E}[\mathcal{L}(\boldsymbol{w}_{k} )-\mathcal{L}_{*}]$ converges faster with the GAF than without the GAF for $k = 2, 3, \cdots$. The proof is complete.
\end{proof}

Theorem \ref{THCond} demonstrates that the GAF alleviates the ill-conditioned problem since the condition number is reduced. Theorem \ref{THConver} shows that under some conditions (similar to strong convexity and gradient Lipschitz), SGD equipped with the GAF converges faster than the original SGD, and it has a linear convergence rate if $M=0$.

\begin{figure*}[!hbt]
    \centering
        \subfigure[Original (3D)]{
            \label{LossSurCom_a}
            \includegraphics[width=1.55in]{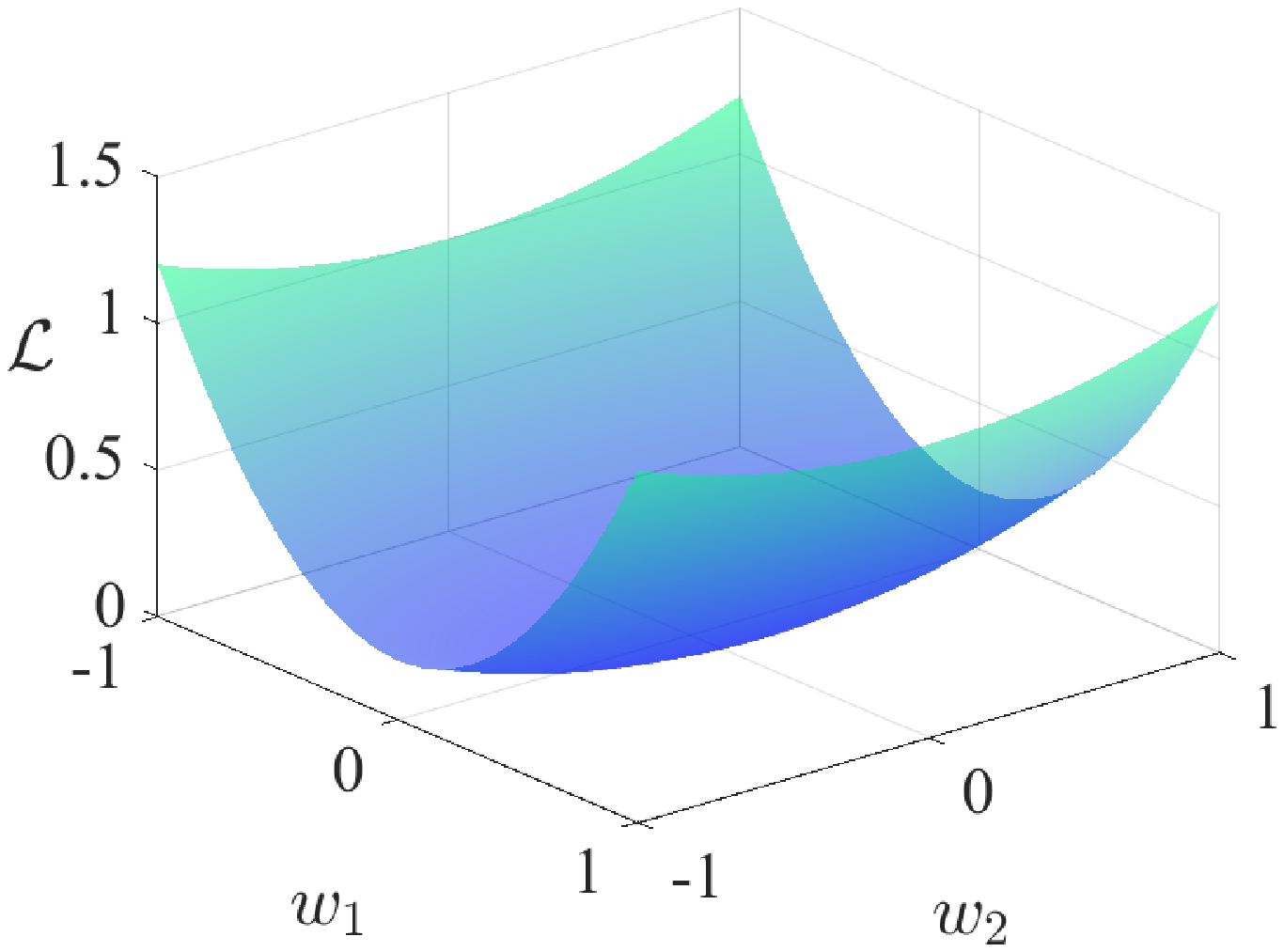}
        }
        \subfigure[Original (contour)]{
            \label{LossSurCom_b}
            \includegraphics[width=1.15in]{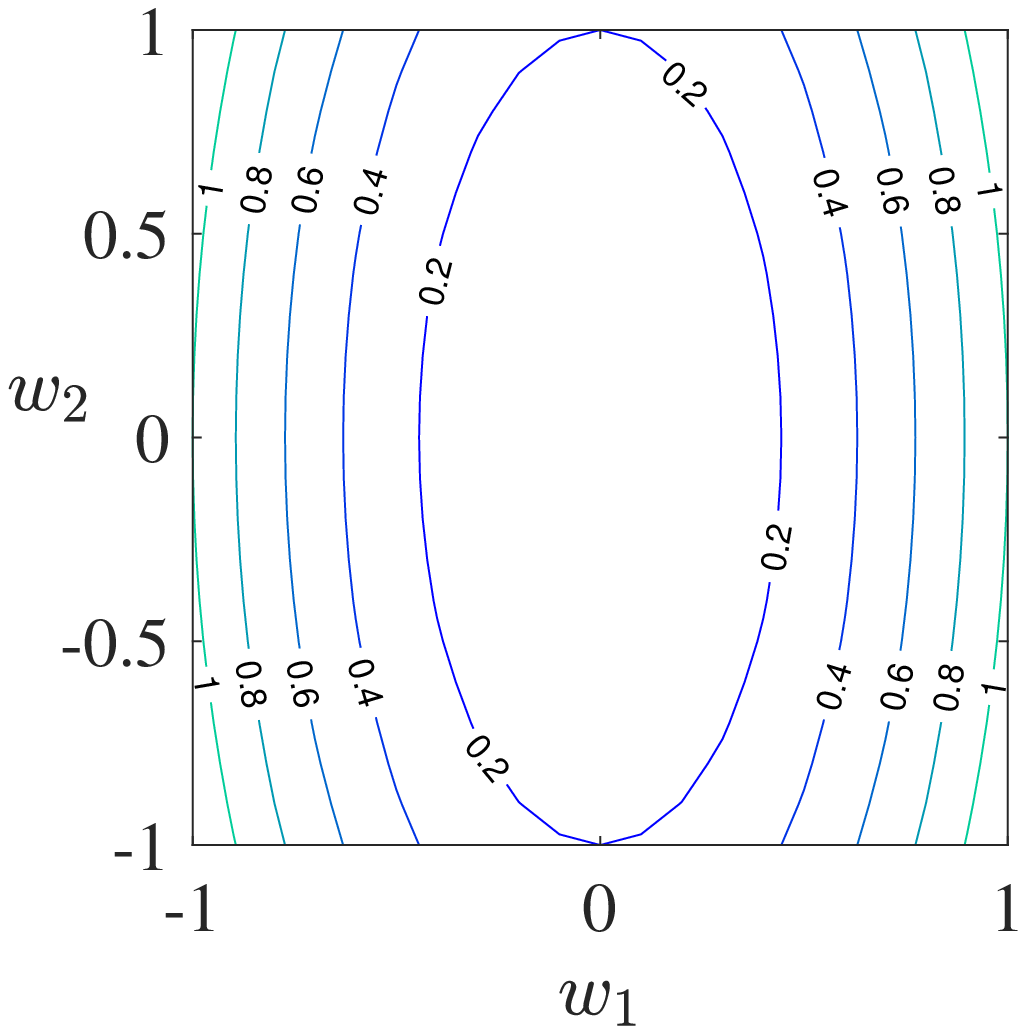}
        }
        \subfigure[Arctan-type GAF (3D)]{
            \label{LossSurCom_c}
            \includegraphics[width=1.55in]{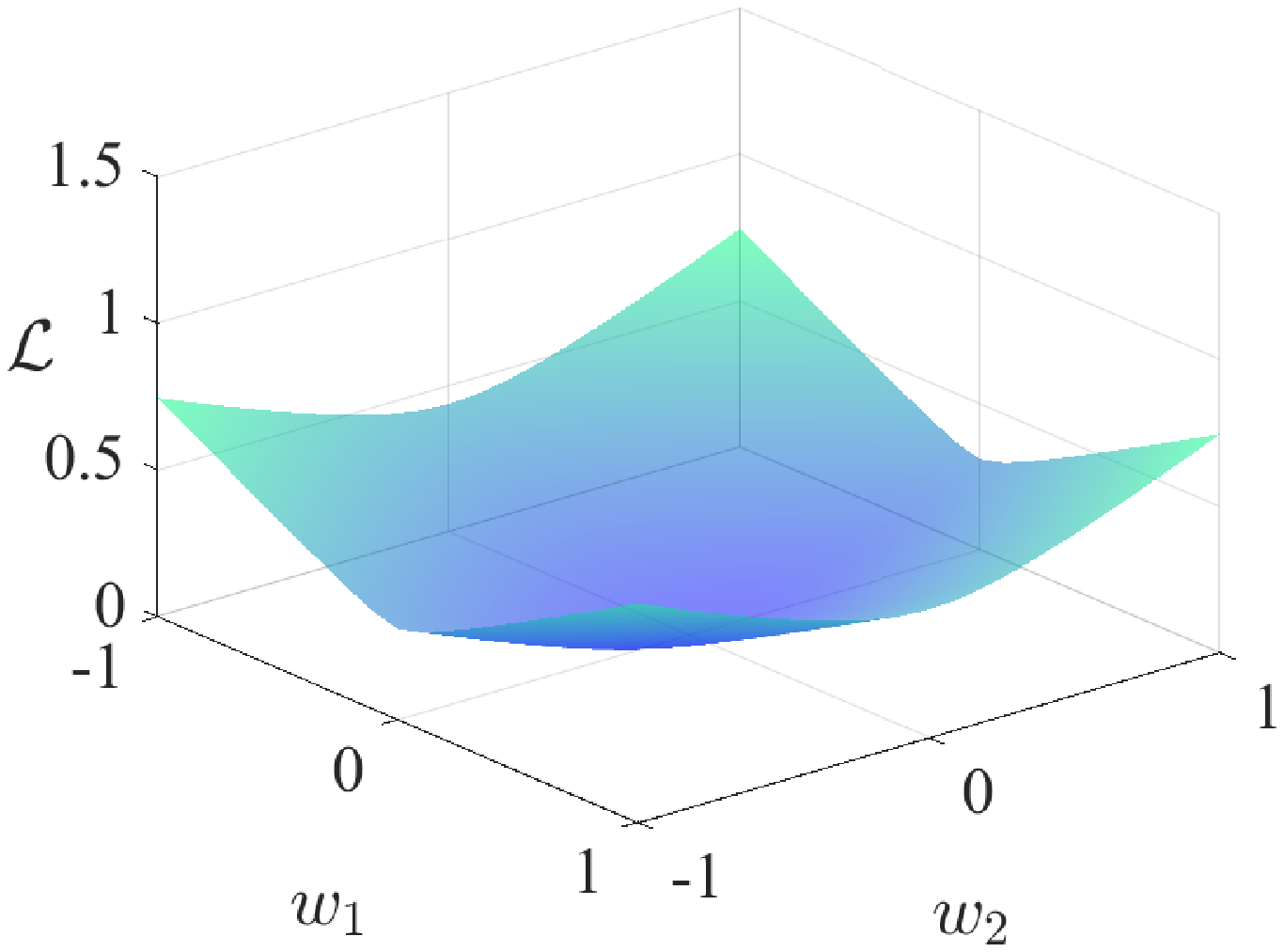}
        }
        \subfigure[Arctan-type GAF (contour)]{
            \label{LossSurCom_d}
            \includegraphics[width=1.15in]{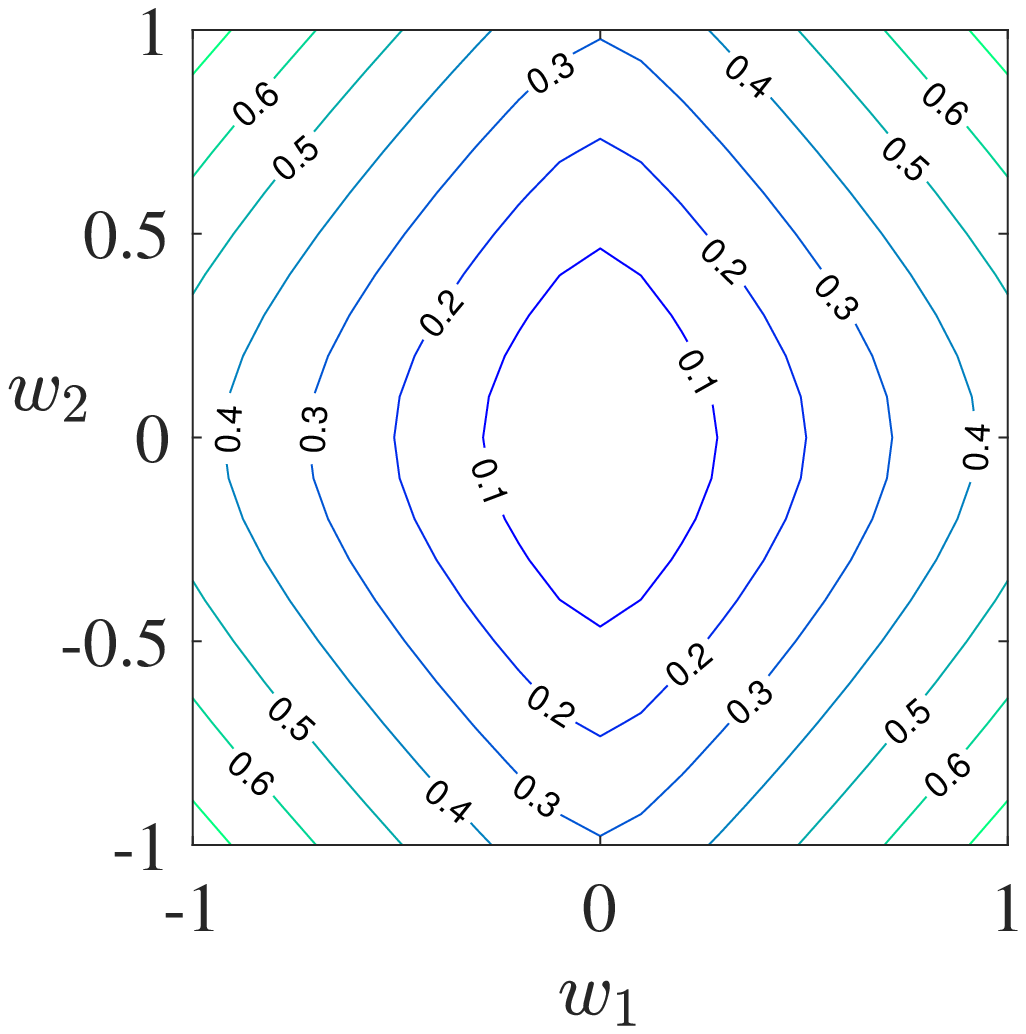}
        }
        \subfigure[Tanh-type GAF (3D)]{
            \label{LossSurCom_e}
            \includegraphics[width=1.55in]{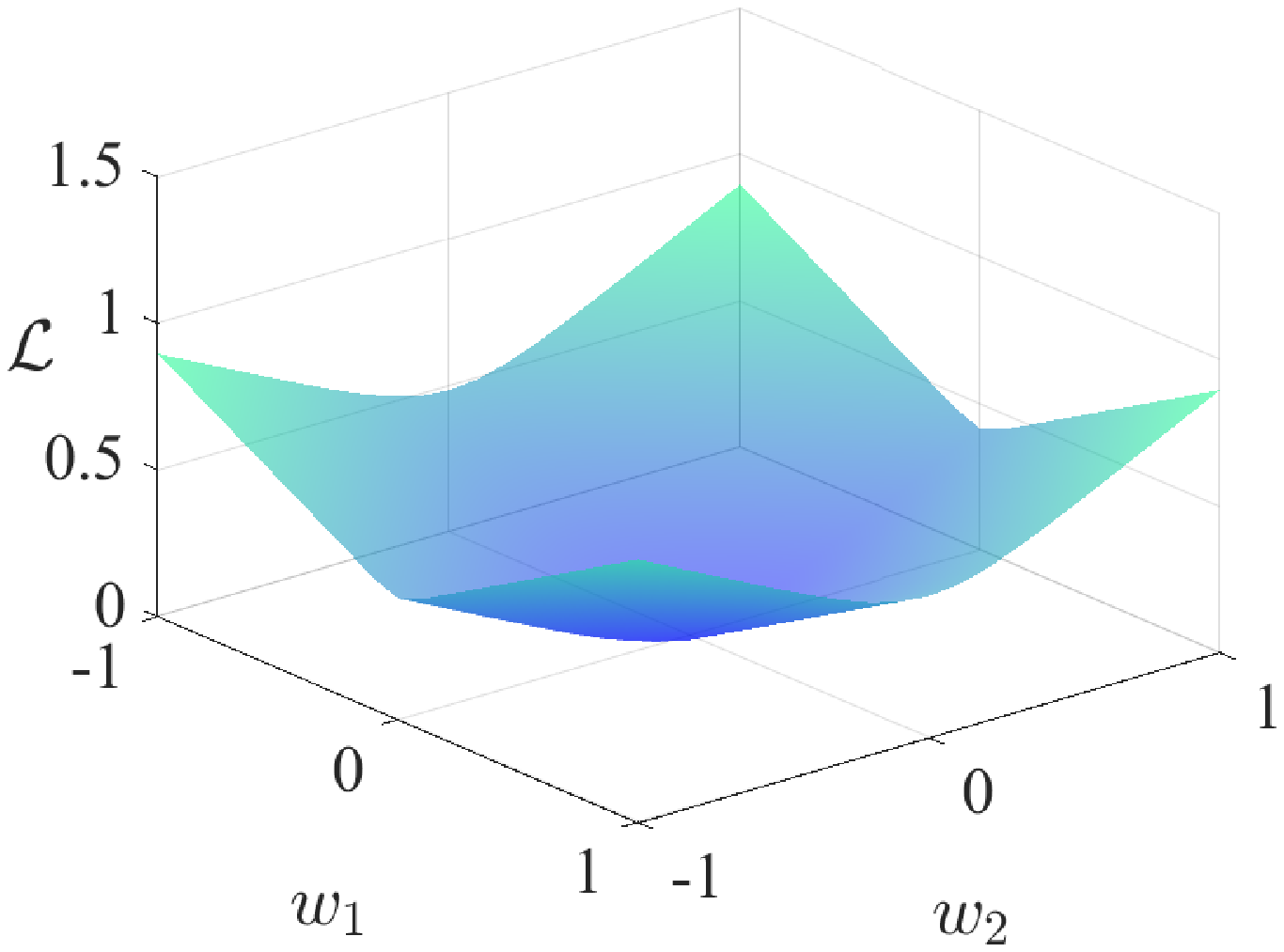}
        }
        \subfigure[Tanh-type GAF (contour)]{
            \label{LossSurCom_f}
            \includegraphics[width=1.15in]{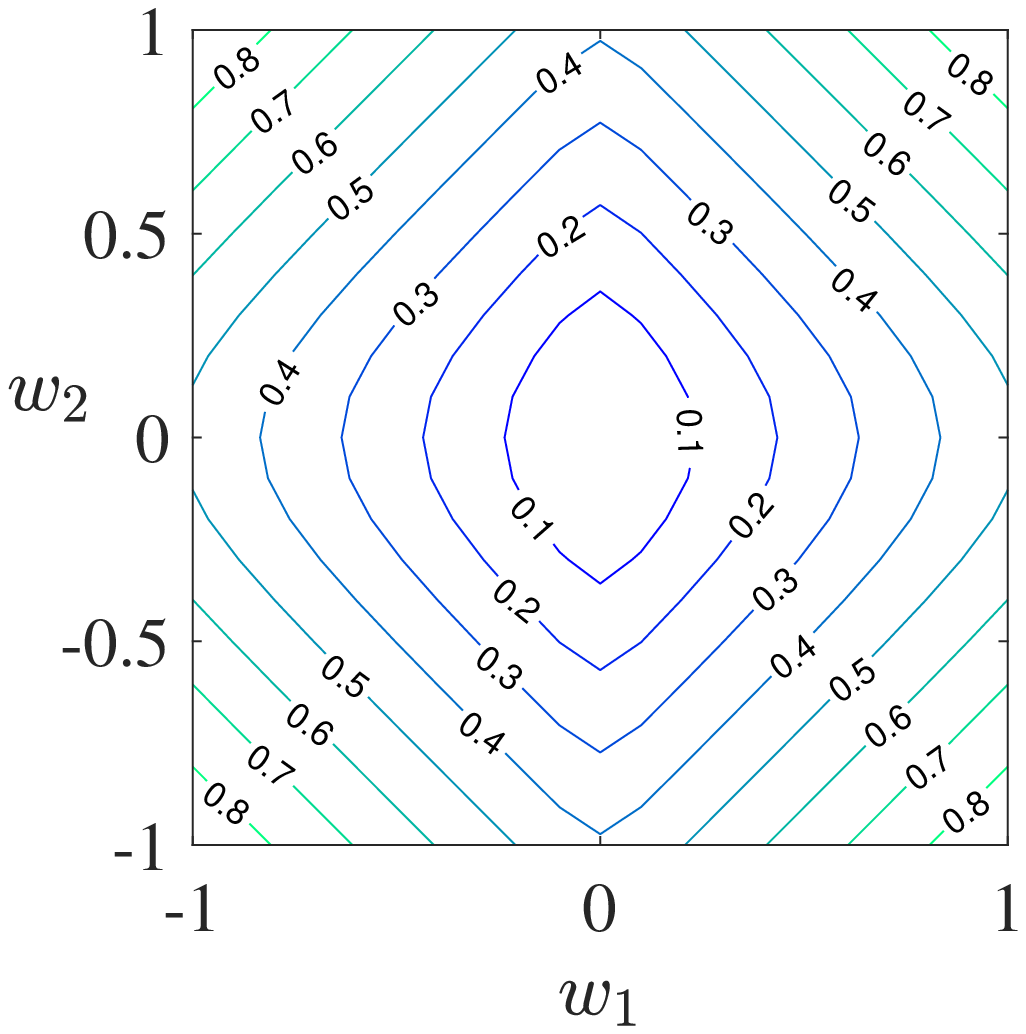}
        }
        \subfigure[Log-type GAF (3D)]{
            \label{LossSurCom_g}
            \includegraphics[width=1.55in]{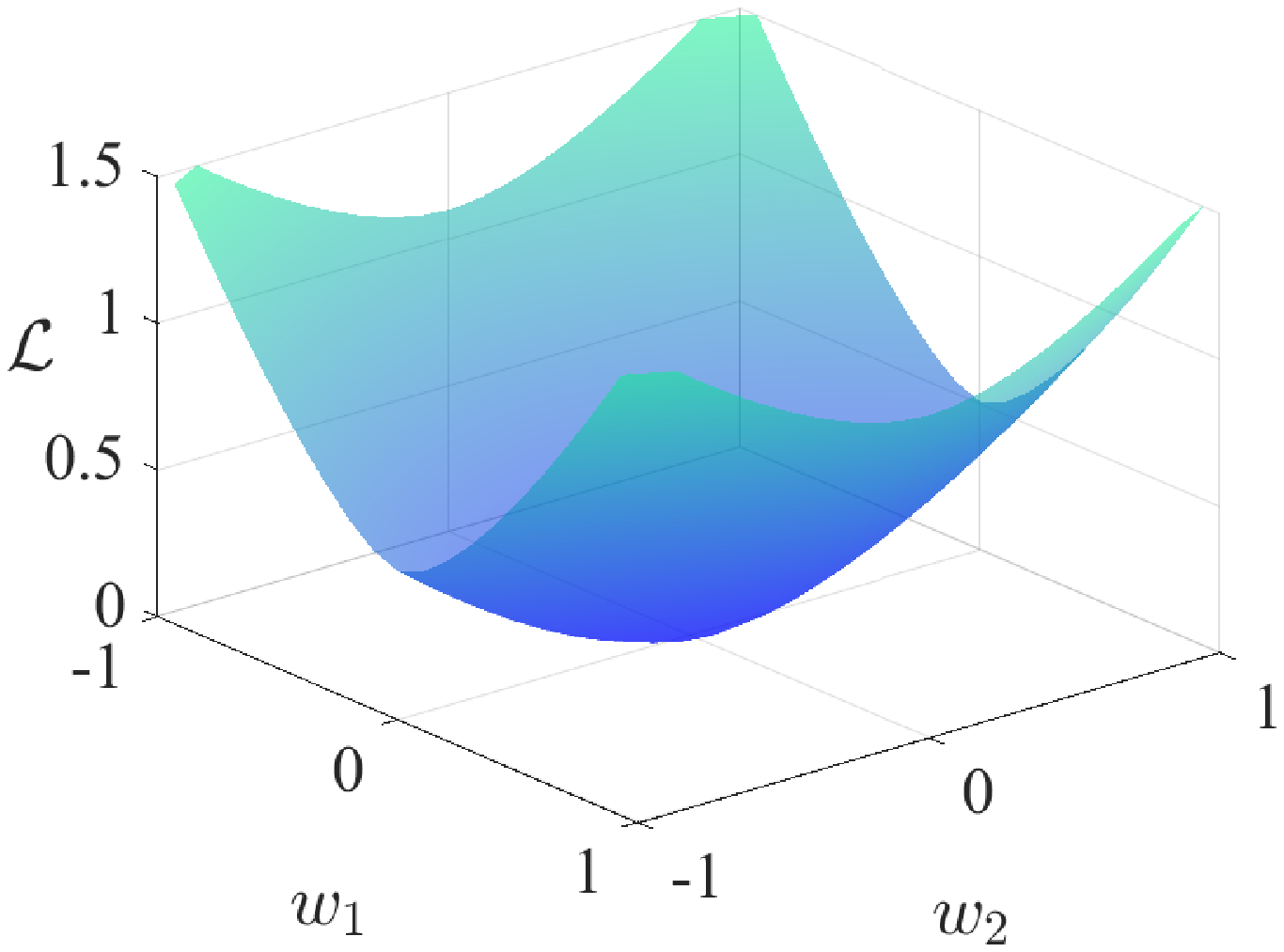}
        }
        \subfigure[Log-type GAF (contour)]{
            \label{LossSurCom_h}
            \includegraphics[width=1.15in]{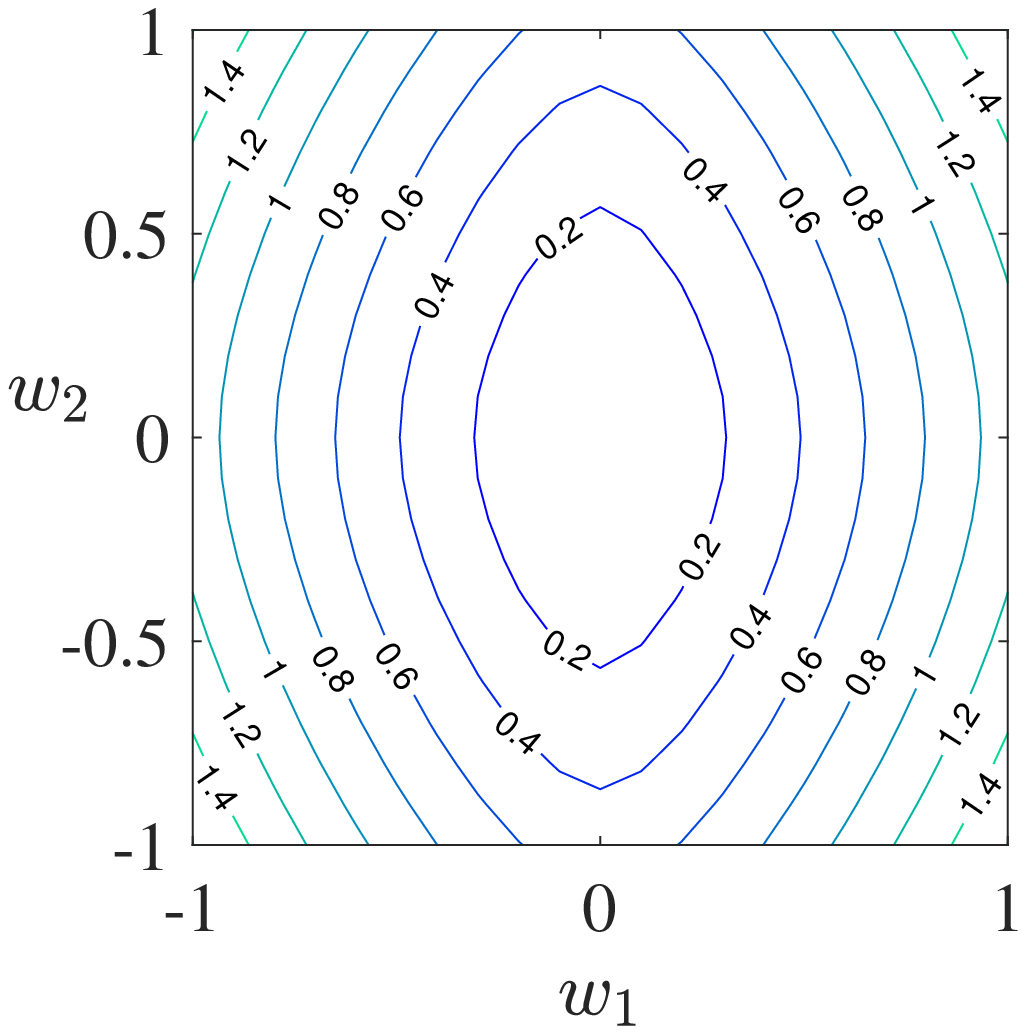}
        }
        \subfigure[Clipping at norm $0.1$ (3D)]{
            \label{LossSurCom_i}
            \includegraphics[width=1.55in]{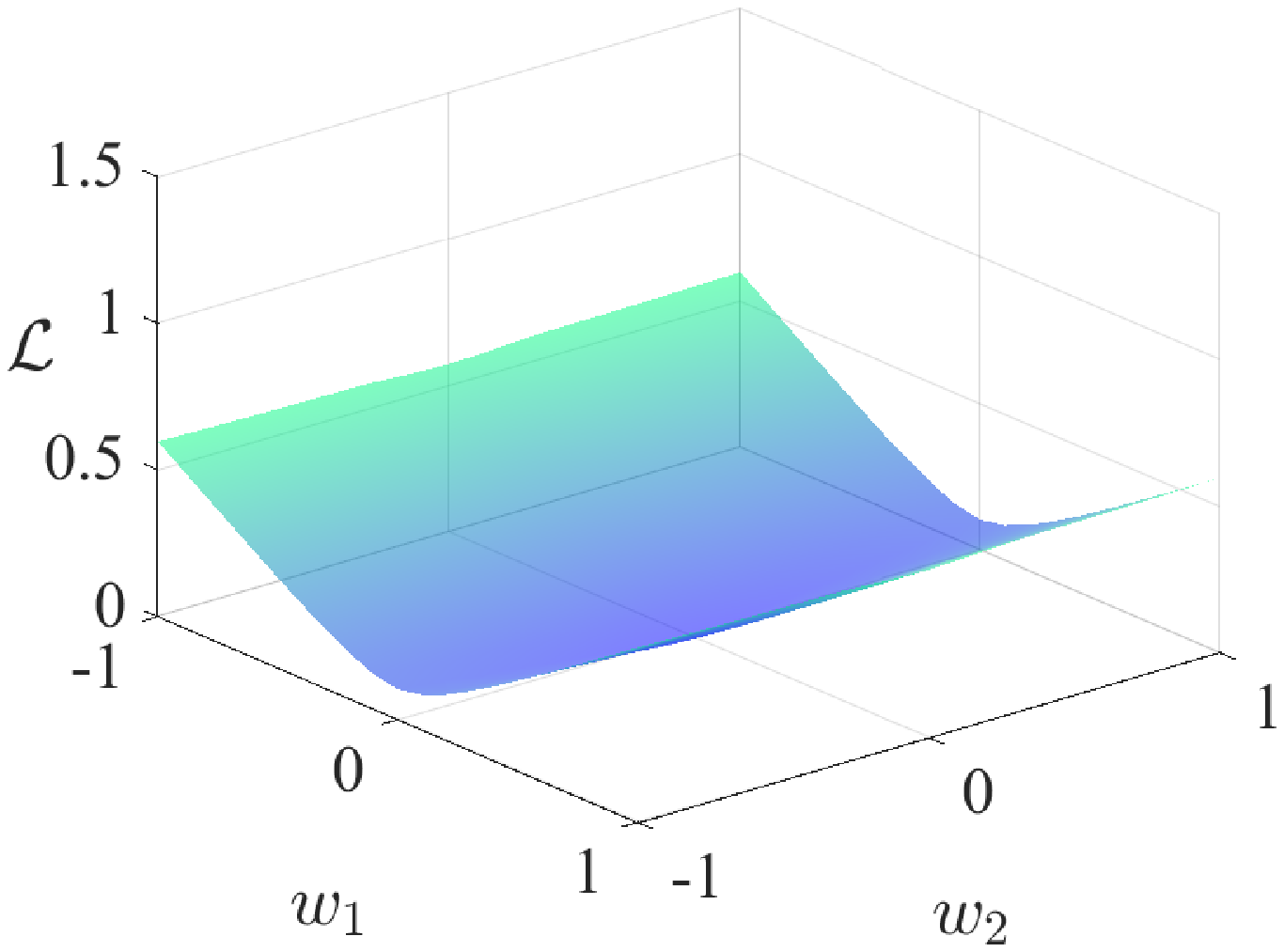}
        }
        \subfigure[Clipping at norm $0.1$ (contour)]{
            \label{LossSurCom_j}
            \includegraphics[width=1.15in]{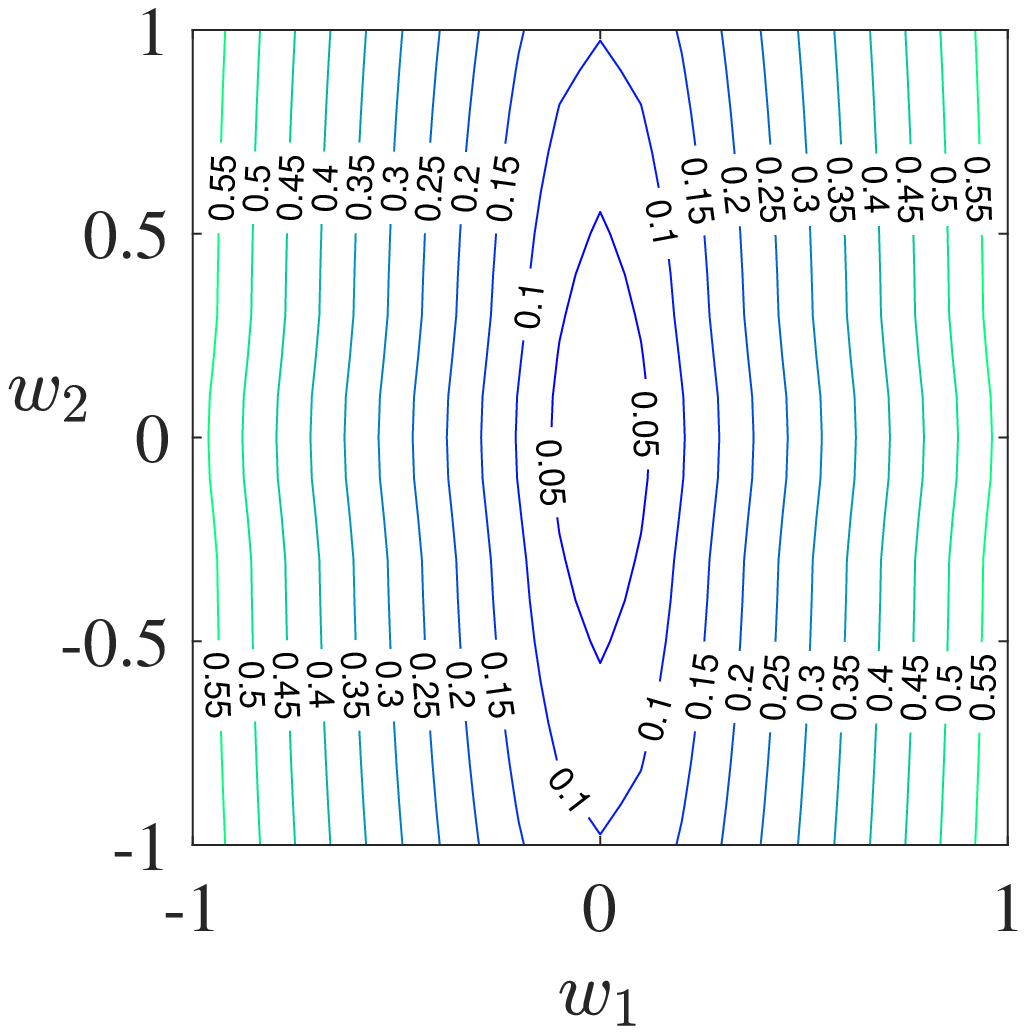}
        }
        \subfigure[Clipping at value $\alpha$ (3D)]{
            \label{LossSurCom_k}
            \includegraphics[width=1.55in]{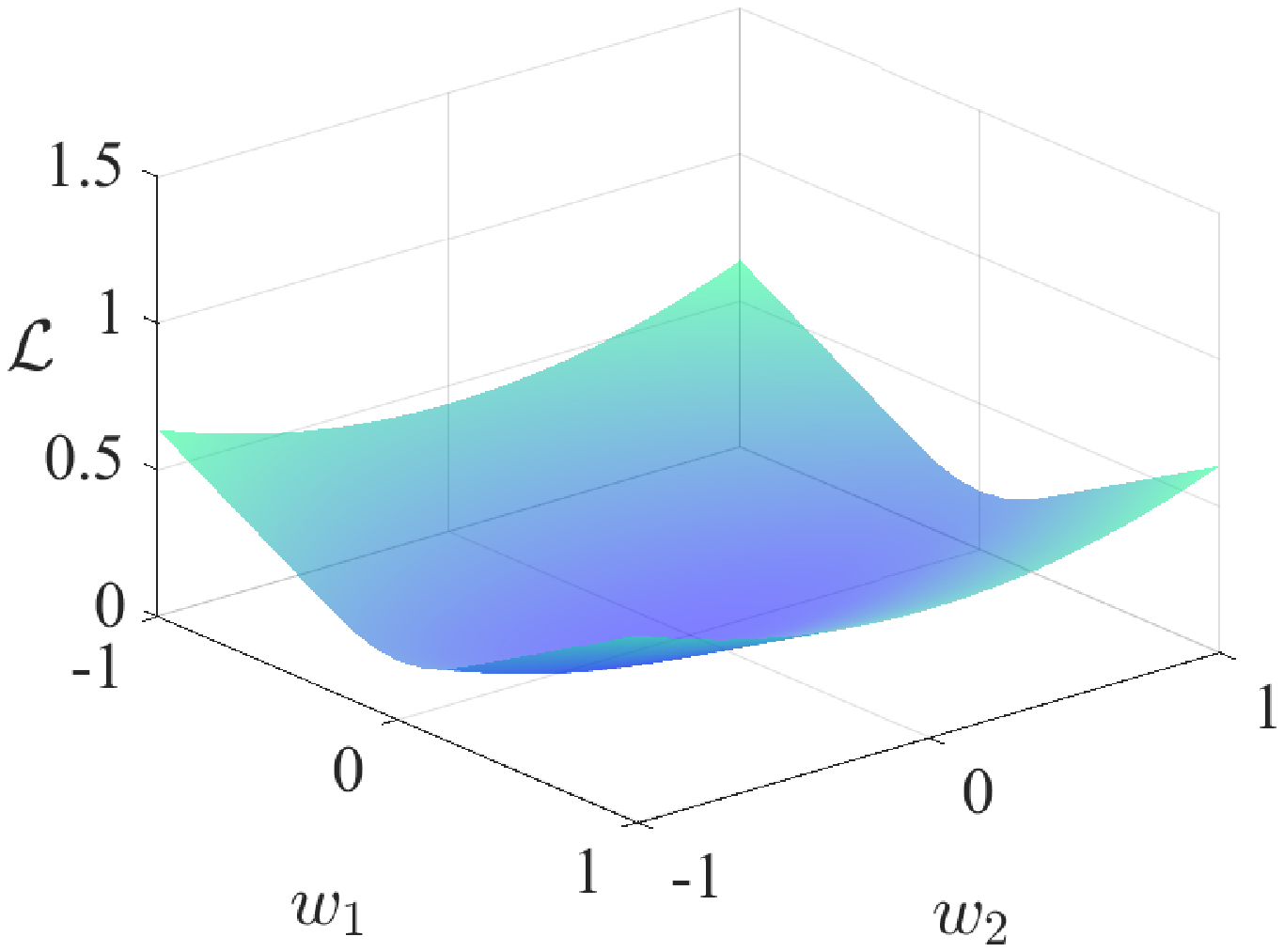}
        }
        \subfigure[Clipping at value $\alpha$ (contour)]{
            \label{LossSurCom_l}
            \includegraphics[width=1.15in]{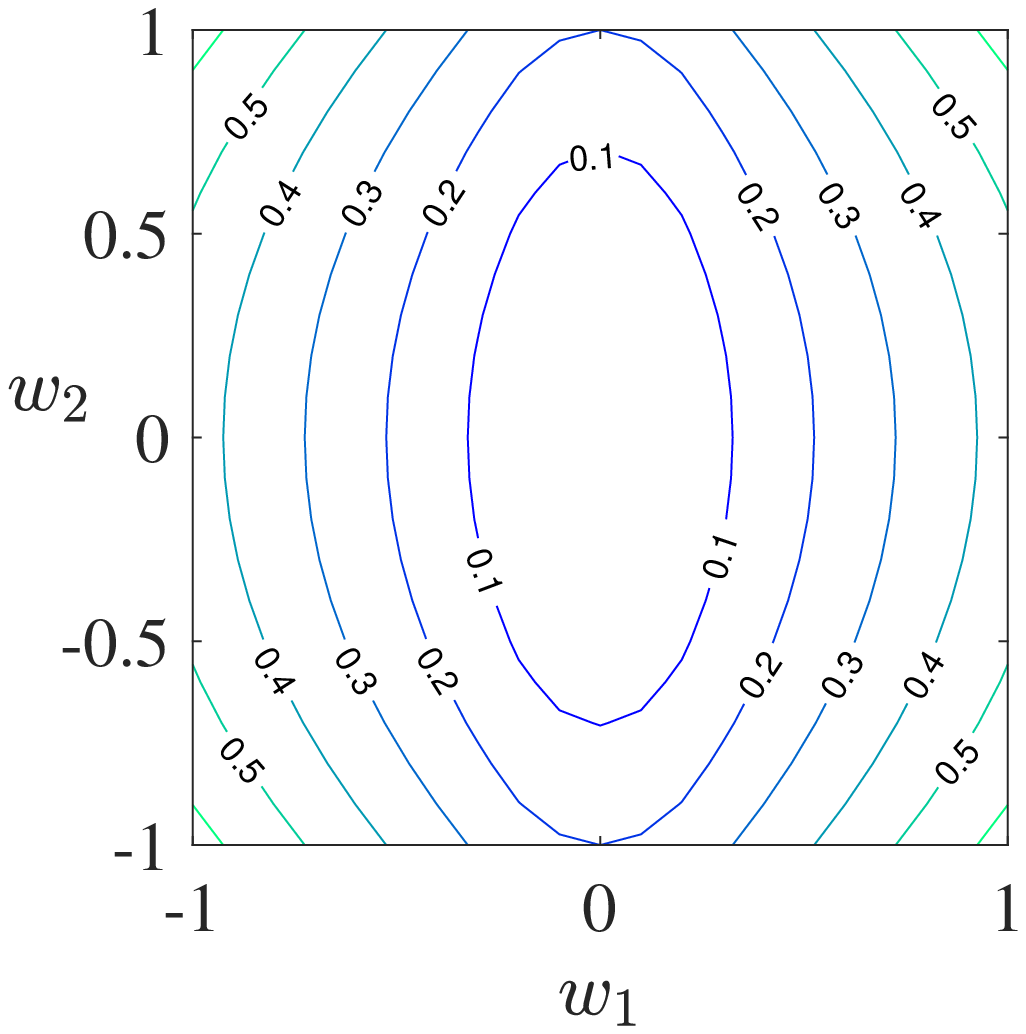}
        }
        \caption{Equivalent loss surfaces of three types of GAFs and two kinds of gradient clipping methods. }
        \label{LossSurCom}
\end{figure*}

After modifying gradients by the GAF, one obtains an equivalent loss surface. The equivalent loss surface refers to a loss surface that does not use a GAF, while it is equivalent to the original loss surface that uses a GAF. With the aid of the auto integral's implementation by symbolic computing framework, in the example shown in Fig. \ref{LossSurCom}, the parameters of the loss surface are set to be separable for simplicity:
\begin{equation}
    \mathcal{L} = w_1^2+0.2 w_2^2.
\end{equation}
In Fig. \ref{LossSurCom}, 3D view and contour view are both provided. Equivalent loss surfaces after acting the GAF are presented in Fig. \ref{LossSurCom}. The original loss surface (Fig. \ref{LossSurCom_a} and Fig. \ref{LossSurCom_b}) is steep/narrow in one direction but flat/wide in another, which means it is an ill-conditioned problem. By using the GAFs, equivalent loss surfaces are significantly less ill-conditioned (see Fig. \ref{LossSurCom_c} through Fig. \ref{LossSurCom_h}). For gradient clipping, the value clipping method slightly suppresses the ill-conditioned problem (Fig. \ref{LossSurCom_i} and Fig. \ref{LossSurCom_j}), while the norm clipping one exacerbates the ill-conditioned degree (Fig. \ref{LossSurCom_k} and Fig. \ref{LossSurCom_l}).

\subsection{Vanishing and Exploding Gradient Problems}\label{sec.3.2}

    This subsection discusses the GAF's effects in vanishing and exploding gradient problems. First, we give some intuitions about why GAF works in these situations, and then take the arctan-type GAF as an example to show the formal evidence. ResNet deals with vanishing and exploding gradient problems by involving identity skip connection, which makes the coefficient in front of the gradient in backpropagation close to $1$ \cite{he2016identity}. This approach eases vanishing/exploding gradient problems to some extent \cite{he2016identity}. The GAF addresses vanishing and exploding gradient problems from a different perspective. Although the gradient may vanish or explode through layers, its sign remains. The GAF enlarges the tiny gradient and restricts the large gradient, and thus avoids vanishing and exploding gradient problems to some extent.

    In order to facilitate the analysis of the determination of hyperparameters (see Section \ref{sec.3.5}), the following theoretical discussions are based on abounded GAF (the arctan-type GAF is chosen as an example, and see Section \ref{sec.3.1} for its formal description). To ease the vanishing gradient problem, we provide the following theorem.

\begin{theorem}[]
    \label{t1}
    Consider an arctan-type GAF. Suppose that ${\alpha {\beta}} > 1$, then $\exists  \epsilon_3  > 0$ that makes $\forall g_n \in \{ (0,  \epsilon_3 ) \cup (- \epsilon_3 , 0) \}$, there is $\vert \acute{g}(g_n)\vert  > \vert g_n \vert$.
\end{theorem}
\begin{proof}
    Define
    $f(g_n) = \acute{g}(g_n) - g_n = \alpha {\arctan}({{\beta}} g_n) - g_n
    $, whose derivatives of the first-order and the second-order can be obtained separately as follows:
    \begin{equation}\label{dGAF}
        f'(g_n) = \frac{{{\alpha \beta}}}{1 +
            ({{\beta}}g_n)^2} - 1,
    \end{equation}
    \begin{equation}\label{ddGAF}
        f''(g_n) = \frac{-2{{\alpha \beta}} ^ 3 g_n}{[1 + ({{\beta}} g_n) ^ 2] ^ 2}.
    \end{equation}
    Evidently, $f'(g_n) > 0$ is guaranteed if ${\alpha{\beta}} > 1$ and $g_n$ is close to $0$, which proves that $f(g_n)$ is a monotonically increasing function. For $g_n$ tending to $0^+$, the limit of $f(g_n)$ can be calculated as
    \[
        {\lim_{g_n \to  0^+} f(g_n)} =
        {\lim_{g_n \to  0^+} \alpha{\arctan}({{\beta}} g_n) - g_n} > f(0) = 0.
    \]
    For $g_n$ tending to $+\infty$, the limit of $f(g_n)$ can be calculated as
    \[
        {\lim_{g_n \to  +\infty} f(g_n)} =
        { \lim_{g_n \to +\infty} \alpha{\arctan}({{\beta}} g_n) - g_n} = -\infty < 0.
    \]
    At this point, it is known that there exists a point $ \epsilon_3  \in (0, +\infty)$ such that $f( \epsilon_3 ) = 0$ holds. Also, recalling Equation \eqref{ddGAF},  one has $f''(g_n) < 0$ for $g_n \in (0, +\infty)$, which means that $-f(g_n)$ is a convex function. Since $f( \epsilon_3 ) = 0$ and $f(0) = 0$, according to Jensen's inequality, for $\tau \in (0,1)$,
    \begin{equation}
        \begin{aligned}
            -f(\tau \cdot 0 + (1-\tau) \epsilon_3  ) & =  -f((1-\tau) \epsilon_3  )                 \\
                                                     & < -\tau f(0) - (1-\tau) f( \epsilon_3 ) = 0.
        \end{aligned}
    \end{equation}
    Then a conclusion is drawn that $f(g_n) > 0$ when $g_n \in (0,  \epsilon_3 )$, which indicates that the following equation holds.
    \[
        f(g_n) = \alpha{\arctan}({{\beta}} g_n) - g_n > 0.
    \]
    Similarly, $f(g_n) = \alpha{\arctan}({{\beta}} g_n) - g_n \in (- \epsilon_3 , 0)$ for $g_n < 0$. So far, it can be concluded that $\vert \acute{g}(g_n)\vert  > \vert g_n \vert$, $\forall g_n \in \{ (0,  \epsilon_3 ) \cup (- \epsilon_3 , 0) \}$. The proof is thus completed.
\end{proof}
One can deduce from Theorem \ref{t1} that if the input gradient is small, the GAF with large ${\alpha {\beta}}$ makes the output gradient value greater than the input one. In other words, the gradient close to $0$ is enlarged, which is an advantageous solution to the vanishing gradient problem. For the exploding gradient problem, the arctan-type GAF restricts the gradient within a certain range since the arctan function is bounded. To sum up, the GAF with tunable factors ${{\alpha}}$ and ${{\beta}}$ is an efficient solution to both the vanishing gradient problem and the exploding gradient problem.

\subsection{Saddle Point Problem}\label{sec.3.3}

In the training of deep neural networks, the saddle point problem is a shackle that limits model performance. The GAF enables the model to escape the saddle point. A three-dimensional loss surface near a saddle point is illustrated in Fig. \ref{FigLossSur_a}. The formal result about the behavior of the GAF near a saddle point is given in what follows.
\begin{figure}[htbp]
    \centering
    \subfigure[Original loss surface]{
        \label{FigLossSur_a}
        \includegraphics[height=1.05in]{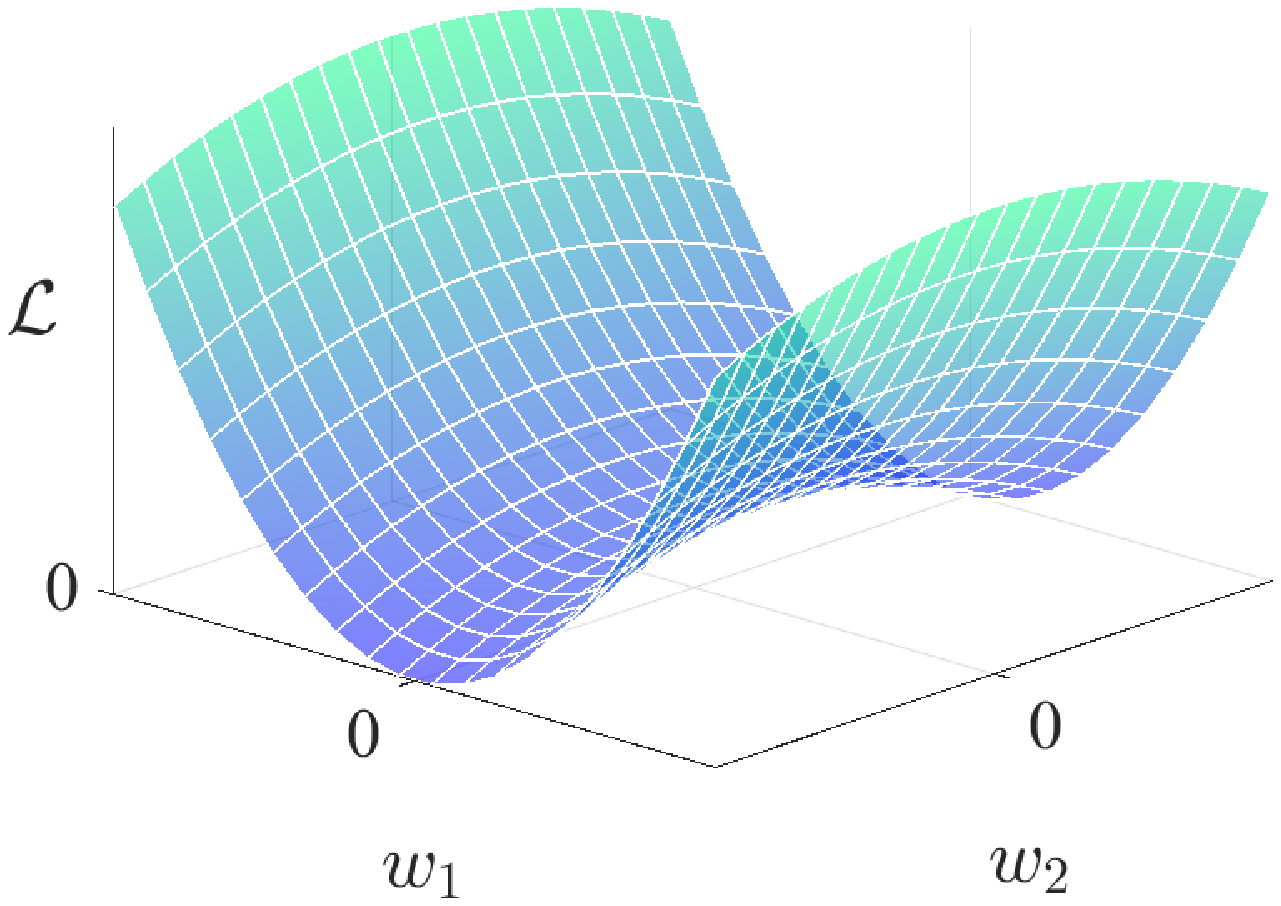}
    }
    \subfigure[Equivalent loss surface after using the GAF]{
        \label{FigLossSur_b}
        \includegraphics[height=1.05in]{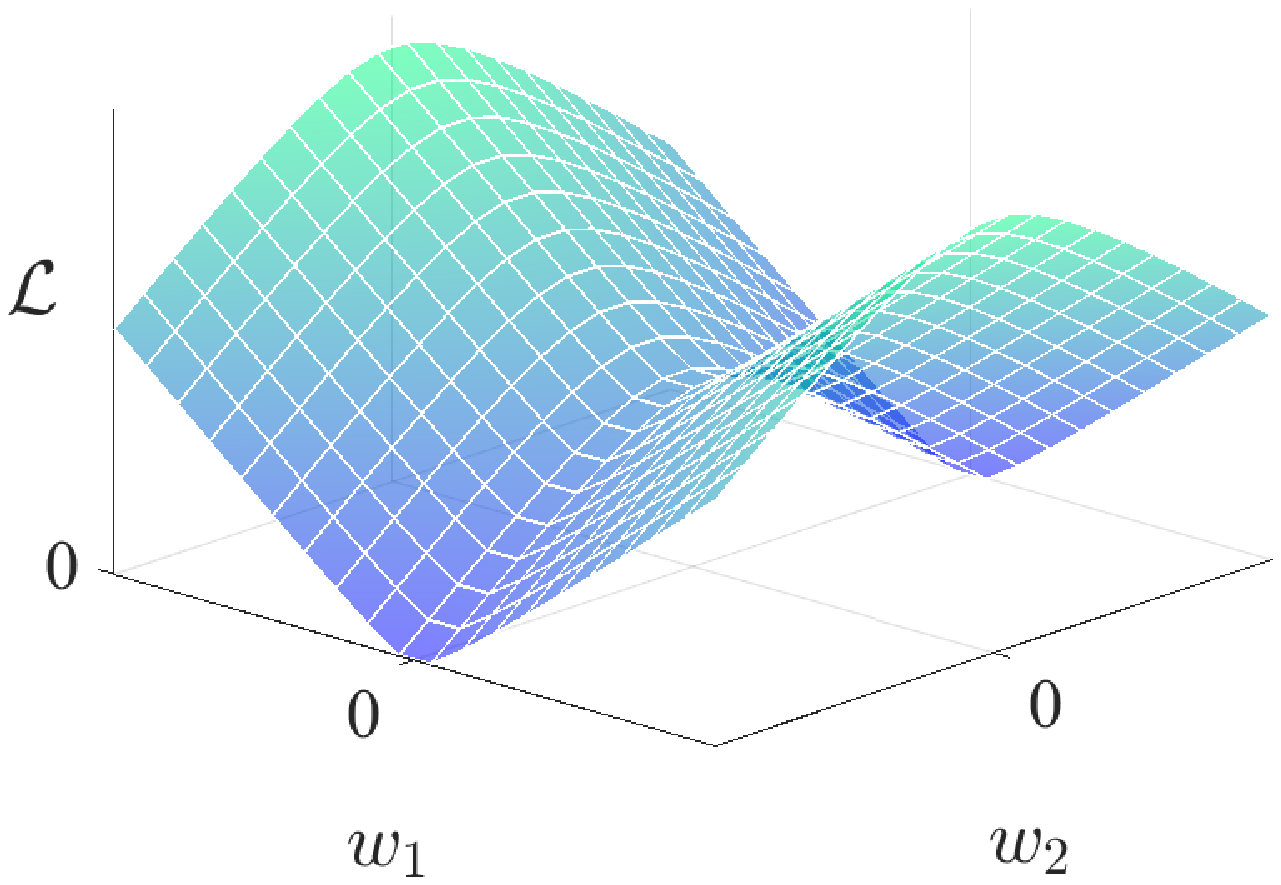}
    }
    \caption{Sketch for the effect of the GAF near a saddle point. }
    \label{FigLossSur}
\end{figure}

\begin{theorem}[]
    \label{t2}
    Consider an arctan-type GAF. $\exists  \epsilon_3  > 0$ and ${\alpha{\beta}} > 1$ that makes $\forall {g}^{\rm O}_i \in \{ (0,  \epsilon_3 ) \cup (- \epsilon_3 , 0) \}$, the GAF-equipped gradient descent optimizer escapes faster than the original one, where $i \in [k+1, k+\delta]$. That is, $\vert {w}^{{\rm O}}_{k+\delta} - {w}^{{\rm O}}_{k} \vert < \vert {w}^{{\rm G}}_{k+\delta} - {w}^{{\rm G}}_{k} \vert$, where the subscript $^{\rm{O}}$ symbolizes the original parameter, the subscript $^{\rm{G}}$ refers to the parameter activated by the GAF; the parameter near the saddle point are denoted as ${w}$, and a total of $\delta$ steps are considered after the $k$th step.

\end{theorem}

\begin{proof}
    The basic form of the gradient descent is ${w}_{k+1} = {w}_{k} - \eta \hat{{g}}_{k}$. Consider a total of $\delta$ steps after the $k$th step, then the gradient descent for both the original and GAF-equipped cases can be respectively written as
    \[
        {w}^{{\rm O}}_{k+\delta} = {w}^{{\rm O}}_{k} - \eta \sum_{i=k}^{\delta} \hat{{g}}^{{\rm O}}_{i}
    \]
    and
    \[
        {w}^{{\rm G}}_{k+\delta} = {w}^{{\rm G}}_{k} - \eta \sum_{i=k}^{\delta} \hat{{g}}^{{\rm G}}_{i},
    \]
    where $\eta$ is the learning rate, ${i}$ means the $i$th step, and $\hat{{g}}$ is the gradient. Since the gradient near the saddle point is close to $0$, based on Theorem \ref{t1}, $\exists  \epsilon_3  > 0$ and ${{\beta}} > 1$ that makes $\forall {g}^{\rm O}_i \in \{ (0,  \epsilon_3 ) \cup (- \epsilon_3 , 0) \}$, there is
    \[
        \vert {g}^{\rm O}_i \vert < \vert {g}^{\rm G}_i \vert.
    \]
    It further yields
    \[
        \left| \eta \sum_{i=k}^{\delta} \hat{{g}}^{{\rm O}}_{i} \right| < \left| \eta \sum_{i=k}^{\delta} \hat{{g}}^{{\rm G}}_{i} \right|.
    \]
    Thus,
    \[
        \vert {w}^{{\rm O}}_{k+\delta} - {w}^{{\rm O}}_{k} \vert < \vert {w}^{{\rm G}}_{k+\delta} - {w}^{{\rm G}}_{k} \vert.
    \]
    The proof is thus completed.
\end{proof}
In order to give an alternative point of view, a sketch of an equivalent loss surface near a saddle point is provided in Fig. \ref{FigLossSur_b}. This figure reveals that the loss surface around a saddle point is deformed to be steeper such that the optimizer escapes faster from the saddle-point region.

\begin{algorithm}[tbp]
    \setstretch{1.35}
    \caption{SGDM with the GAF}
    \label{alg}
    \begin{algorithmic}\label{sgdmgaf}
        \REQUIRE
        $\alpha$, $\beta$: factors required by the GAF.
        \REQUIRE
        $\boldsymbol{w}^0$: initial value of the weight vector.\\
        \REQUIRE
        $\mathcal{M} (\boldsymbol{x}^{(i)}; \boldsymbol{w}^{(i)})$: the prediction of the trained model on input samples $\boldsymbol{x}^{(i)}$ for each iteration.\\
        \REQUIRE
        $\mathcal{L}$: the loss function; $\eta$: the learning rate; $\mu_{\rm{m}}$: momentum factor; $\boldsymbol{\upsilon}^0$: initial velocity.
        $k \leftarrow 0$
        \WHILE{$\mathcal{L}$ is not small enough}
        \STATE{
        $k \leftarrow k + 1$\\
        $\kappa$ samples \{$\boldsymbol{x}^{(1)},\cdots,\boldsymbol{x}^{(\kappa)}$\} are selected from the training set whose corresponding labels are  $ \{ y^{(1)},\cdots,y^{(\kappa)} \}$\\
        $\boldsymbol{\breve{g}}_{k} \leftarrow \nabla_{\boldsymbol{w}}{\sum}_{i=1}^{\kappa} \mathcal{L}(\mathcal{M} (\boldsymbol{x}^{(i)}; \boldsymbol{w}_{k}), \boldsymbol{y}^{(i)}) / \kappa$\\
        $\boldsymbol{{g}}_{k}\leftarrow \mu_{\rm{m}} \boldsymbol{{g}}_{k - 1} - \eta \boldsymbol{\breve{g}}_{k}$\\
        ${\acute{\boldsymbol{g}}}_{k} \leftarrow \acute{\boldsymbol{g}} (\boldsymbol{{g}}_{k})$ \\
        $\boldsymbol{w}_{k} \leftarrow \boldsymbol{w}_{k} + {\acute{\boldsymbol{g}}}^{(k)}$\\
        }
        \ENDWHILE
        \RETURN $\boldsymbol{w}_{k}$
    \end{algorithmic}
\end{algorithm}

\section{Experiments}\label{sec.4}

In this section, implementation of the GAF, hyperparameter determination process, experiment settings, and comparative experimental results are elaborated in detail.

\subsection{Implementation}\label{sec.3.4}

This subsection explains how the GAF is embedded in an optimizer.
Suppose that there are $\kappa$ samples \{$x^{(1)},\cdots,x^{(\kappa)}$\} in a training batch. The corresponding labels for the aforementioned training samples are \{$y^{(1)},\cdots,y^{(\kappa)}$\}; $\mathcal{M} (\boldsymbol{x}^{(i)}; \boldsymbol{w}^{(i)})$ is the output of the trained model on input samples $\boldsymbol{x}^{(i)}$ for each iteration, where $\boldsymbol{w}^{(i)}$ is the weight vector. Take SGDM as an example, and let $\eta$ stand for the learning rate in an optimization process and $\mu_{\rm{m}}$ be the coefficient of momentum in the SGDM optimizer. The implementation of the GAF embedded in SGDM is as follows. First, compute the gradient $\boldsymbol{g}$ in the stochastic training process:
\begin{equation}
    \boldsymbol{\breve{g}}_{k} = \nabla_{\boldsymbol{w}}{\sum}_{i=1}^{\kappa} \mathcal{L}(\mathcal{M} (\boldsymbol{x}^{(i)}; \boldsymbol{w}^{(i)}), \boldsymbol{y}^{(i)}) / \kappa.
\end{equation}

Then, perform a standard momentum step, and replace the original gradient with the GAF-activated one. The detailed process is presented in Algorithm \ref{sgdmgaf}. For other gradient-based optimizers, such as SGD and Adam, we directly act the GAF on the gradient. Due to its simple implementation, GAF is convenient to be nested into modern deep learning frameworks such as PyTorch and TensorFlow. The cost of implementing the GAF algorithm is merely adding a few lines to the code of the SGDM optimizer, and the cost of additional calculations is negligible since it is just an element-wise activation on the gradient with the computational complexity $\mathcal{O} {(\check{n})}$.
For example, it costs only $1.2$ s extra training time in one epoch on CIFAR-$100$ with ResNet-$50$ model in our experiments ($42$ s for one epoch).

\subsection{Hyperparameter Determination}\label{sec.3.5}
\begin{figure}[bth]
    \centering
    \subfigure[ResNet-$18$, CIFAR-$10$]{
        \label{CurveRes18}
        \includegraphics[height=0.9in]{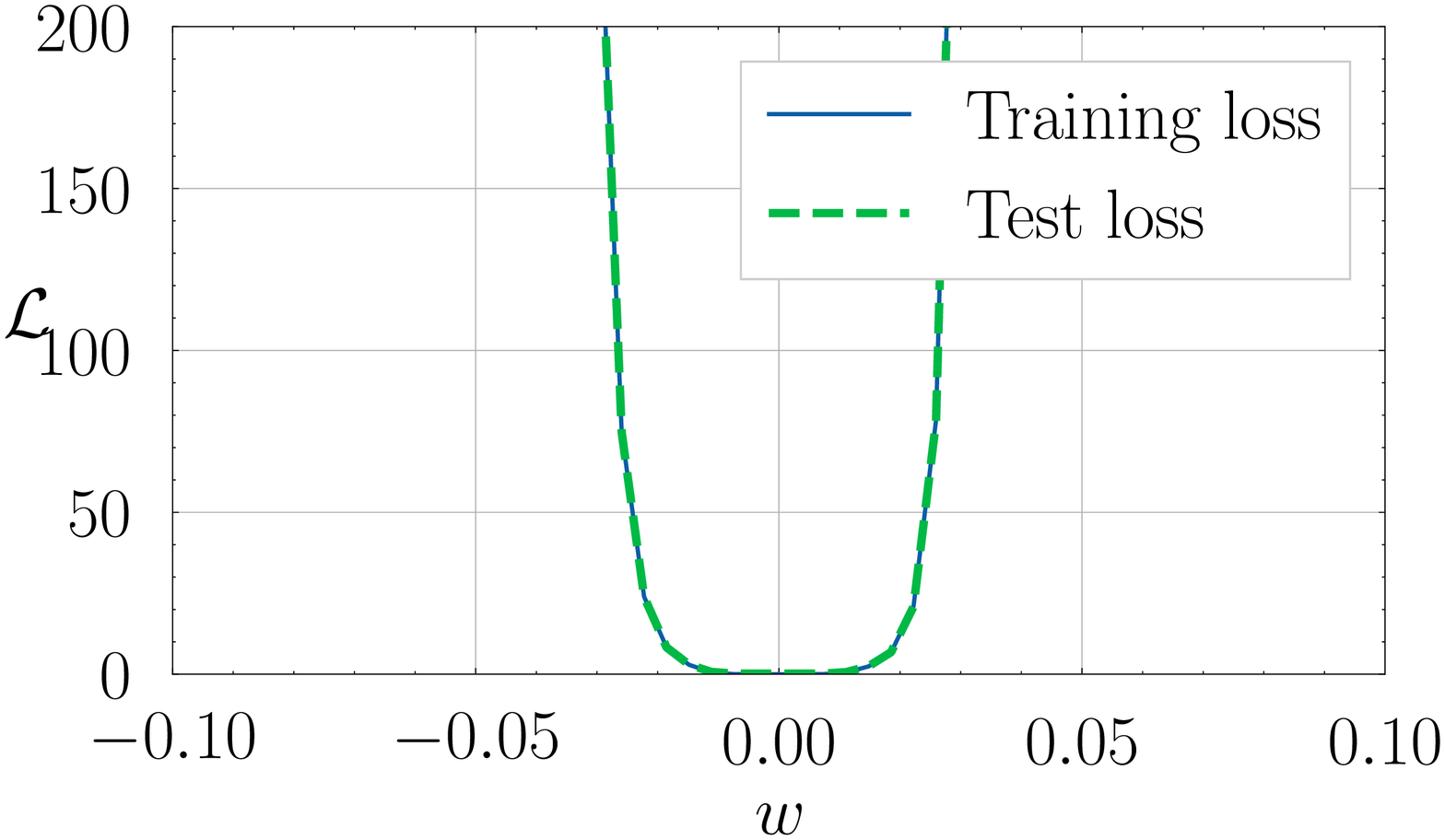}
    }
    \subfigure[Enlarged view of (a)]{
        \label{CurveRes18}
        \includegraphics[height=0.9in]{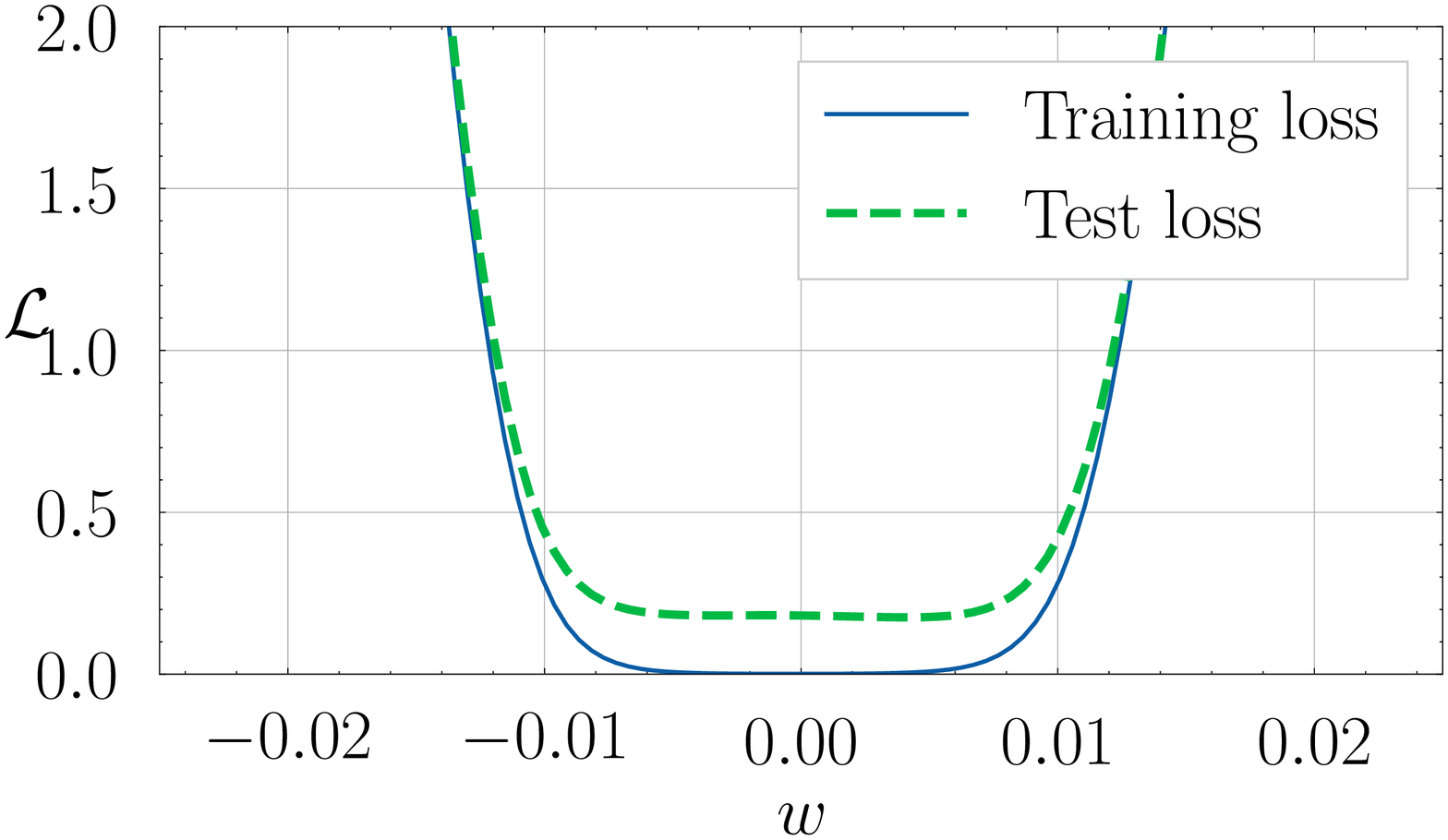}
    }
    \subfigure[DenseNet ($k=12$, depth$=121$), CIFAR-$100$]{
        \label{CurveDense}
        \includegraphics[height=0.9in]{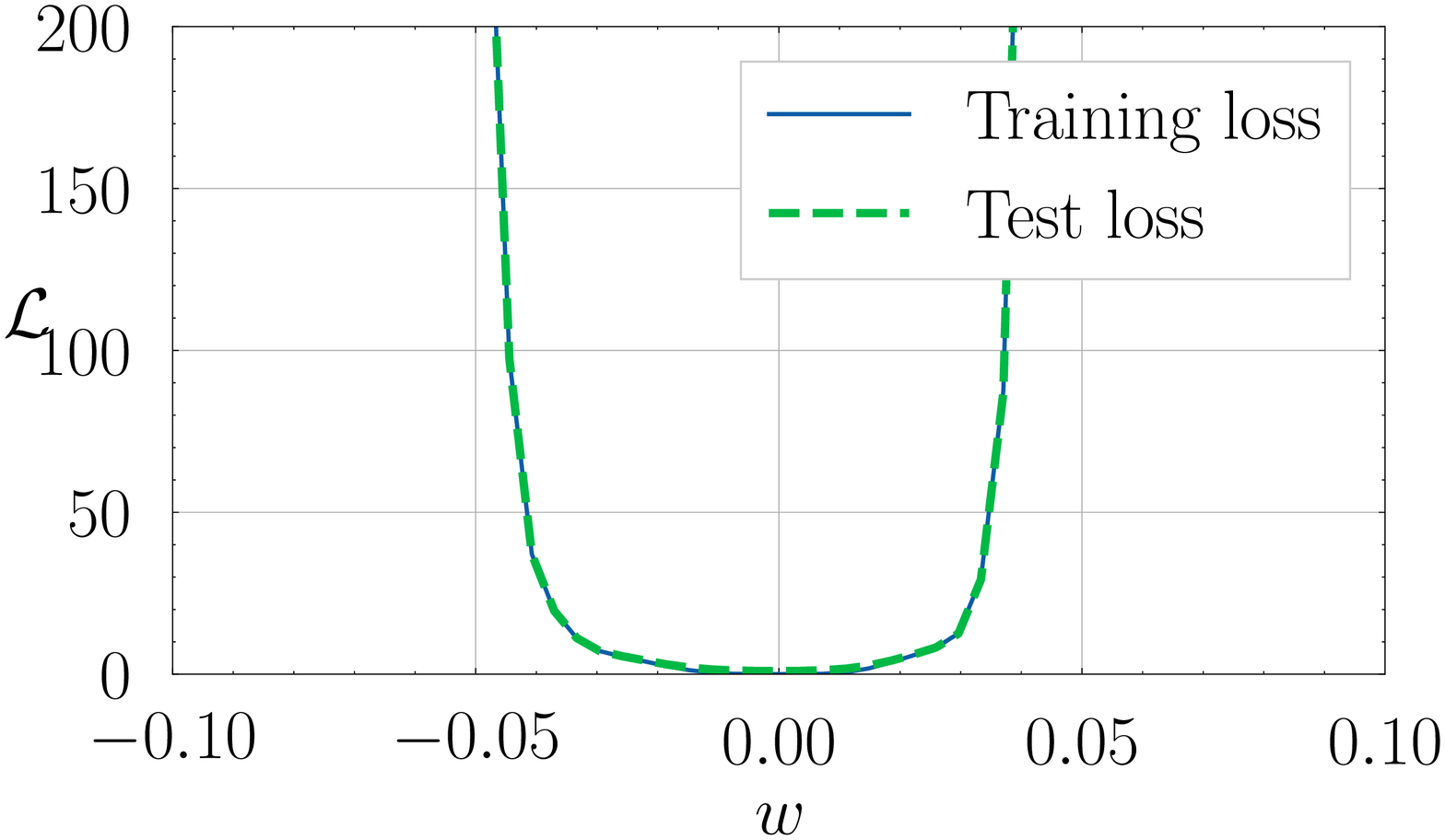}
    }
    \subfigure[Enlarged view of (c)]{

        \label{CurveDense}
        \includegraphics[height=0.9in]{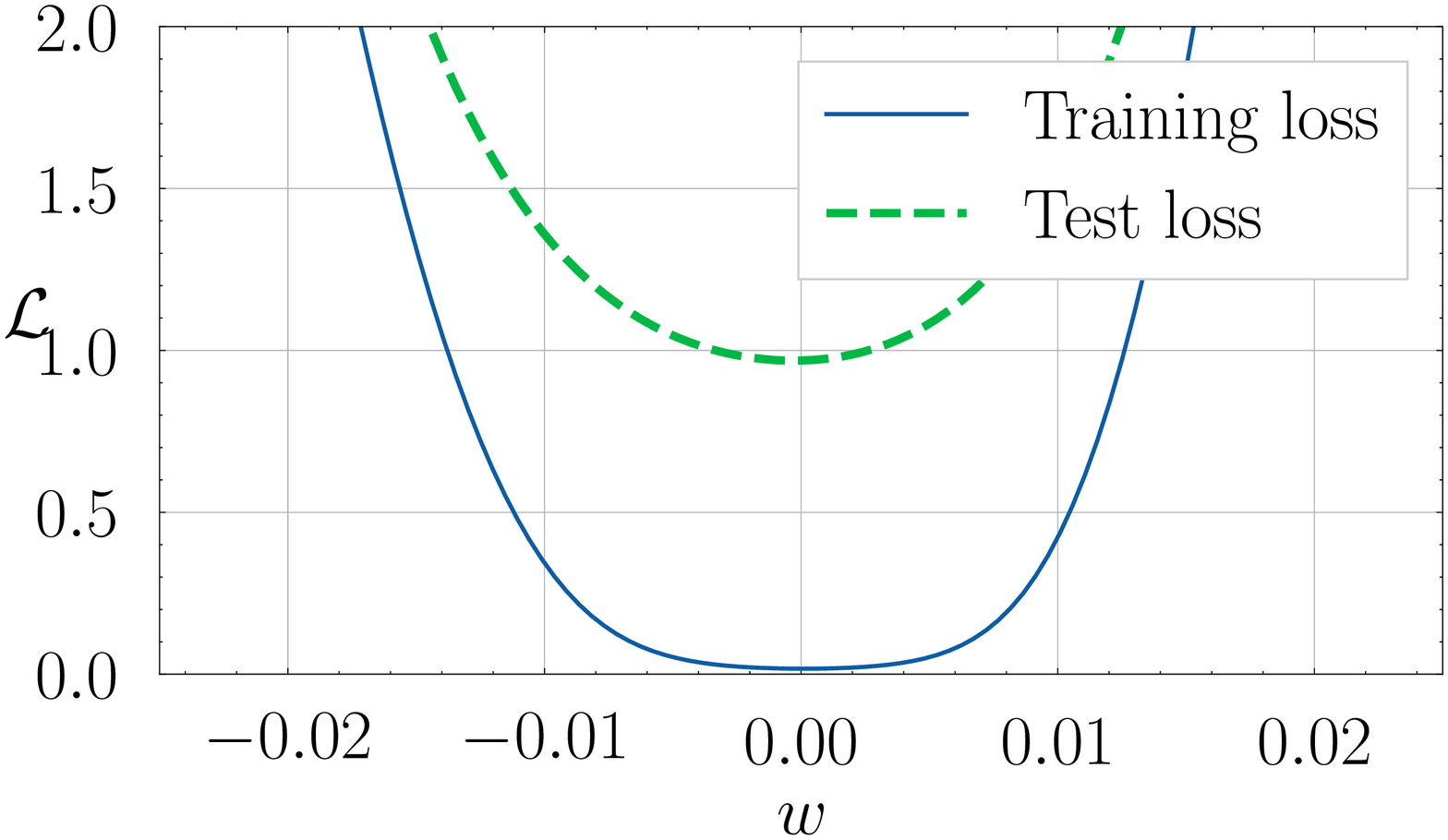}
    }
    \caption{Realistic loss curves near the optimum point after training on CIFAR. }
    \label{Curve}
\end{figure}

\begin{figure}[bth]
    \centering{\includegraphics[width=2.8in]{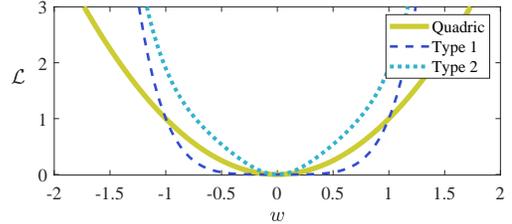}}
    \caption{Sketch of different types of loss curves. }
    \label{FigSimCurve}
\end{figure}

    Generally, the values of $\alpha$ and $\beta$ are determined according to the range of gradient distribution in the training process of the original model (this rule can be relaxed; see the end of this subsection). There are several steps to determine the hyperparameters involved in the GAF.
    \begin{itemize}
        \item[(a)] Train an original network without using the GAF and record the maximal gradient in the training process.
        \item[(b)] Draw the loss curve according to the visualization method provided in \cite{li2018visualizing} and determine what kind of GAF should be used. Examples of this visualization method are given in Fig. \ref{Curve}. Among all kinds of loss curves, the quadric curve is optimal from the perspective of optimization due to the following reason. Theoretically, for a quadric curve, once the optimal step size $1/\ell$ is taken, no matter where the initial point is, the gradient descent optimizer takes only one step to reach the minimum, where $\ell$ is the Lipschitz constant of the gradient. Specifically, for regions in the loss curve that is flatter than the quadric curve, we increase the gradient (Type 1 with the loss $\mathcal{L} <1$ in Fig. \ref{FigSimCurve}); for regions sharper than the quadric curve, we decrease the gradient (Type 2 with the loss $\mathcal{L} <1$ in Fig. \ref{FigSimCurve}). For regions with near-zero gradient, the first case requires $\alpha \beta>1$ for the arctan-type GAF, while the latter one requires $\alpha \beta<1$.
        \begin{figure}[hbt]
            \centering{\includegraphics[width=3.7in]{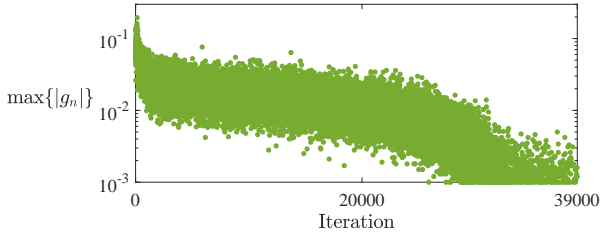}}
            \caption{Maximal gradient magnitudes during the training process of ResNet-$18$ on CIFAR-$10$. }
            \label{FigMaxGrad}
        \end{figure}
        \item[(c)] For regions with large gradient Lipschitz constant (e.g., regions with $\mathcal{L} >2$ for both Type 1 and Type 2 in Fig. \ref{FigSimCurve}), it requires that the large gradient is restrained. In practice, we observed that the region with a relatively large gradient exists in many datasets and models. The first basis is that if we dramatically increase the learning rate at the beginning of the training process, the training generally fails. The second basis is given in Fig. \ref{Curve}. For ResNet-$18$ on CIFAR-$10$ and DenseNet on CIFAR-$100$, the corresponding loss curves are with large-gradient regions. Besides, as can be seen in Fig. \ref{Curve}, restraining the large-gradient region makes the curve to be closer to the quadric shape. The third basis is given in Fig. \ref{FigMaxGrad}. In the training process of ResNet-$18$ on CIFAR-$10$, the maximal gradient magnitude exceeds $0.1$. Considering the weight is generally with a small magnitude, the maximal gradient magnitude is relatively large.
        \item[(d)] According to the maximal gradient magnitude during the training process, set an $\alpha$ so as to restrain the large gradient. The reason for setting such an $\alpha$ is to make sure that the GAF bounds the gradients within a limited range so that the exploding gradient problem and the ill-conditioned problem are alleviated.
    \end{itemize}

    In practice, the shape of the loss curve is usually similar to the Type 1 in Fig. \ref{FigSimCurve}, as shown in Fig. \ref{Curve}. Theoretically, for a bounded GAF (e.g., arctan or tanh), setting an $\alpha$ to restrict the maximal gradient value and a $\beta$ to make $\alpha \beta>1$ is beneficial to the Type 1 in Fig. \ref{FigSimCurve}. As presented in Section III in the revised manuscript, this setting is able to address the vanishing/exploding gradient problems, saddle point problems, and ill-conditioned problems in one shot. In practice, there is no need to draw the loss curve for each individual model-dataset pair. One reason is that the computational cost is high. Another reason is that we find that there are two hyperparameter sets ($\{\alpha=0.1, \beta = 20\}$ and $\{\alpha=0.2, \beta = 10\}$) that are robust across a variety of experiments such as image classification (see Tables \ref{TBImageNetCom} and \ref{TBGAFs}) and object detection (see Table \ref{TBOD}). Since training the original model from scratch is also computationally expensive and the maximal gradient magnitude usually appears in the early stage of the training process, it is unnecessary to go through the whole training process. In practical applications, it is suggested that firstly try $\{\alpha=0.1, \beta = 20\}$ and $\{\alpha=0.2, \beta = 10\}$ (just like $0.1$ is often the first attempt for the learning rate). If they work not well, then follow the steps (a) through (d).

\subsection{Setup of Experiments}\label{sec.4.1}

    The comparison experiment is conducted on ImageNet, CIFAR-$100$, CIFAR-$10$, and PASCAL VOC datasets. ImageNet is a large-scale image classification dataset. There are $1.28$ million and $50,000$ images in the training set and validation set of ImageNet, respectively. Following \cite{swish, elu, adam}, for experiments on ImageNet, validation set is used for testing with single crop of size $224 \times 224$. On ImageNet, mix precision training is utilized, and all models are trained for $150$ epochs with batch size $64$, initial learning rate $0.1$, momentum $0.9$, and weight decay $0.00004$. Cosine annealing with the warmup for $5$ epochs is used as the learning rate scheduler.
    CIFAR-$100$ and CIFAR-$10$ both consist of $60,000$ images, for a total of $50,000$ training images and $10,000$ test images with size of $32 \times 32$. In addition, the training epoch for experiments conducted on CIFAR is $200$. The momentum is set to be $0.9$, and the weight decay is specified as $0.0005$. The initial learning rate is set as $0.1$, and the cosine annealing \cite{CosineAnnealing} is employed as the learning rate decay strategy. Besides, the batch size is set as $256$ for CIFAR.
    For the object detection task, the PASCAL VOC dataset is used. Specifically, training and validation sets of both PASCAL VOC 2007 and PASCAL VOC 2012 are used for training, and the test set of PASCAL VOC 2007 is used for testing. The metrics of the performance are average precision (AP) and mean of AP (mAP). On PASCAL VOC, all models are trained for $120\,000$ iterations with batch size $32$, initial learning rate $0.001$, momentum $0.9$, and weight decay $0.0005$. Besides, cosine annealing with the warmup for $500$ iterations is used as the learning rate scheduler.
    Moreover, the loss function leveraged in all experiments is cross-entropy. No additional tricks are used in these experiments. All models are trained from scratch to achieve fairness as much as possible. All experimental programs are performed in Python $3.7.10$ and PyTorch $1.8.1$ framework. The GPUs on which the experiments depended are $10$ RTX 2080 Ti (for ImageNet and PASCAL VOC), $2$ RTX $3090$ (for CIAFR-$100$ and CIFAR-$10$), and $2$ Quadro RTX $8000$ (for CIAFR-$100$ and CIFAR-$10$).

\begin{table*}[!hbt]
    \centering
    \caption{ Test Accuracies of SGDM on ImageNet with and without GAF}
        \begin{tabular}{@{}lccc@{}}
            \toprule
            Model                              & Setting            & Top-1 validation accuracy (\%) & Top-5 validation accuracy (\%) \\ \midrule
            ResNet-$18$                        & -                  & 69.67                          & 89.00                          \\
            ResNet-$18$   w/ value clipping    & Threshold = 0.1    & 70.14 (+0.47)                         & 89.29 (+0.29)                         \\
            ResNet-$18$   w/ norm clipping     & Threshold = 0.1    & 69.11 (-0.56)                         & 88.55 (-0.45)                          \\
            ResNet-$18$   w/ GAF               & Arctan $(0.1, 20)$ & \textbf{70.71 (+1.04)}                     & \textbf{89.57 (+0.57)}                     \\
            ResNet-$18$   w/ GAF               & Arctan $(0.2, 10)$ & 70.41 (+0.74)                          & 89.47 (+0.47)                         \\ \midrule
            ResNet-$34$                        & -                  & 73.03                          & 90.79                          \\
            ResNet-$34$   w/ value clipping    & Threshold = 0.1    & 72.43 (-0.60)                         & 90.92 (+0.13)                         \\
            ResNet-$34$   w/ norm clipping     & Threshold = 0.1    & 72.61 (-0.42)                         & 90.70 (-0.09)                         \\
            ResNet-$34$   w/ GAF               & Arctan $(0.1, 20)$ & \textbf{73.64 (+0.61)}                     & 91.30 (+0.51)                         \\
            ResNet-$34$   w/ GAF               & Arctan $(0.2, 10)$ & 73.49 (+0.46)                         & \textbf{91.34 (+0.55)}                     \\ \midrule
            ResNet-$50$                        & -                  & 75.22                          & 92.15                          \\
            ResNet-$50$   w/ value clipping    & Threshold = 0.1    & 74.90 (-0.32)                         & 92.04 (-0.11)                         \\
            ResNet-$50$   w/ norm clipping     & Threshold = 0.1    & 75.15 (-0.07)                         & 92.24 (+0.09)                         \\
            ResNet-$50$   w/ GAF               & Arctan $(0.1, 20)$ & 75.91 (+0.69)                         & 92.60 (+0.45)                         \\
            ResNet-$50$   w/ GAF               & Arctan $(0.2, 10)$ & \textbf{76.02 (+0.80)}                     & \textbf{92.64 (+0.49)}                     \\ \midrule
            SE-ResNet-$18$                     & -                  & 70.63                          & 89.65                          \\
            SE-ResNet-$18$   w/ value clipping & Threshold = 0.1    & 70.78 (+0.15)                         & 89.58 (-0.07)                         \\
            SE-ResNet-$18$   w/ norm clipping  & Threshold = 0.1    & 69.89 (-0.74)                         & 88.99 (-0.66)                         \\
            SE-ResNet-$18$   w/ GAF            & Arctan $(0.1, 20)$ & 71.15 (+0.52)                         & 90.14 (+0.49)                          \\
            SE-ResNet-$18$   w/ GAF            & Arctan $(0.2, 10)$ & \textbf{71.54 (+0.91)}                     & \textbf{90.24 (+0.59)}                     \\ \midrule
            SE-ResNet-$34$                     & -                  & 72.75                          & 90.94                          \\

            SE-ResNet-$34$   w/ value clipping & Threshold = 0.1    & 73.05 (+0.30)                         & 90.97 (+0.03)                         \\
            SE-ResNet-$34$   w/ norm clipping  & Threshold = 0.1    & 72.77 (+0.02)                         & 90.91 (-0.03)                          \\

            SE-ResNet-$34$   w/ GAF            & Arctan $(0.1, 20)$ & 73.64 (+0.89)                         & \textbf{91.67 (+0.73)}                     \\
            SE-ResNet-$34$   w/ GAF            & Arctan $(0.2, 10)$ & \textbf{73.74 (+0.99)}                     & 91.50 (+0.56)                         \\ \midrule
            SE-ResNet-$50$                     & -                  & 75.57                          & 92.25                          \\
            SE-ResNet-$50$   w/ value clipping & Threshold = 0.1    & 75.49 (-0.08)                              & 92.32 (+0.07)                              \\
            SE-ResNet-$50$   w/ norm clipping  & Threshold = 0.1    & 75.33 (-0.24)                         & 92.37 (+0.12)                          \\
            SE-ResNet-$50$   w/ GAF            & Arctan $(0.1, 20)$ & \textbf{76.32 (+0.75)}                     & \textbf{92.96 (+0.71)}                     \\
            SE-ResNet-$50$   w/ GAF            & Arctan $(0.2, 10)$ & 76.30 (+0.73)                         & 92.95 (+0.70)                         \\ \bottomrule
        \end{tabular}
        \label{TBImageNetCom}
\end{table*}

\subsection{Comparison Experiments}\label{sec.4.2}

\begin{table*}[htb]
    \centering
        \caption{Test Accuracies of SGDM on CIFAR-$100$ with and without different GAFs}
        \resizebox{\textwidth}{!}{%
            \begin{tabular}{@{}lccccccc@{}}
                \toprule
                Model           & w/o GAF & Tanh (0.1, 20) & Tanh (0.2, 10) & Log (0.1, 20)  & Log (0.2, 10) & Arctan (0.1, 20) & Arctan (0.2, 10) \\ \midrule
                DenseNet-CIFAR  & 76.22   & 76.74 (+0.52)         & \textbf{77.30 (+1.08)} & 76.90 (+0.68)         & 76.72 (+0.50)         & 76.90 (+0.68)           & 76.91 (+0.69)           \\
                DenseNet-$169$  & 79.99   & \textbf{80.71 (+0.72)} & 80.29 (+0.30)         & 80.70 (+0.71)          & 80.40 (+0.41)        & 80.37 (+0.38)            & 79.77 (-0.22)           \\
                DPN-$26$        & 79.05   & \textbf{79.97 (+0.92)} & 79.90 (+0.85)         & 79.63 (+0.58)         & 79.76 (+0.71)        & 79.50 (+0.45)           & 79.26 (+0.21)           \\
                GoogLeNet       & 79.78   & 80.35 (+0.57)         & \textbf{80.56 (+0.78)} & 80.18 (+0.40)         & 80.43 (+0.65)         & 80.25 (+0.47)           & 80.26 (+0.48)           \\
                ResNet-$50$     & 79.54   & 80.43 (+0.89)         & 80.11 (+0.57)         & \textbf{80.58 (+1.04)} & 80.49 (+0.95)        & 79.73 (+0.19)           & 79.90 (+0.36)           \\
                RegNetX-$200$MF & 78.10   & 78.52 (+0.42)         & 78.48 (+0.38)         & \textbf{78.71 (+0.61)} & 78.53 (+0.43)        & 78.26 (+0.16)           & 78.14 (+0.04)           \\
                \bottomrule
            \end{tabular}%
        }
        \label{TBGAFs}
\end{table*}

\begin{table*}[htb]
    \centering
        \caption{Test Accuracies of Adam on CIFAR with and without the GAF}
        \begin{tabular}{@{}lcccccc@{}}
            \toprule
            \multicolumn{1}{l}{\multirow{2}{*}{Model}} & \multicolumn{3}{c}{CIFAR-$10$} & \multicolumn{3}{c}{CIFAR-$100$}                                                                                      \\ \cmidrule(l){2-4}  \cmidrule(l){5-7}
            \multicolumn{1}{c}{}                       & w/o GAF                        & Arctan (0.1, 20)                & Arctan (0.2, 10)       & w/o GAF & Arctan (0.1, 20)       & Arctan (0.2, 10)       \\ \midrule
            ResNet-$34$                                & 94.14                          & 94.36 (+0.22)                   & \textbf{94.44 (+0.30)} & 74.57   & \textbf{75.07 (+0.50)} & 75.06 (+0.49)          \\
            Pre-Act-ResNet-$34$                        & 93.63                          & \textbf{93.97 (+0.34)}          & 93.96 (+0.33)          & 73.48   & \textbf{73.75 (+0.27)} & 73.69 (+0.21)          \\
            ResNet-$50$                                & 93.74                          & 94.41 (+0.67)                   & \textbf{94.42 (+0.68)} & 75.01   & 75.34 (+0.33)          & \textbf{75.69 (+0.68)} \\
            Pre-Act-ResNet-$50$                        & 94.06                          & 94.15 (+0.09)                   & \textbf{94.40 (+0.34)} & 73.73   & \textbf{74.82 (+1.09)} & 74.71 (+0.98)          \\
            \bottomrule
        \end{tabular}
        \label{TBAdam}
\end{table*}
\begin{table*}[!h]
    \centering
    \caption{Test Accuracies of SGDM on PASCAL VOC 2007 with and without GAF}
    \resizebox{\textwidth}{!}{
        \setlength{\tabcolsep}{0.4mm}{
                    \begin{tabular}{lccccccccccccccccccccccccc}
                        \toprule
                        Model            & Setting            & ~ ~mAP~ ~  & aero       & bicycle    & bird       & boat       & bottle     & bus        & car        & cat        & chair      & cow        & table      & dog        & horse      & mbike      & person     & plant      & sheep      & sofa       & train      & tv         \\ \midrule
                        SSD300 w/o   GAF & -                  & 77.71      & 82.57      & \textbf{85.47} & 75.54      & 70.24      & 52.13      & 85.81      & 86.37      & \textbf{88.32} & 60.85      & 81.7       & \textbf{77.02} & \textbf{85.02} & 86.82      & 84.17      & 79.57      & \textbf{52.34} & 77.41      & 79.15      & 86.22      & 77.38      \\
                        SSD300 w/   GAF  & Arctan $(0.1, 20)$ & 78.01      & 82.81      & 83.97      & 76.64      & 68.85      & 53.28      & \textbf{86.72} & \textbf{87.14} & 86.79      & \textbf{64.32} & \textbf{82.86} & 75.37      & 83.82      & \textbf{88.12} & 84.20      & 80.30      & 52.07      & 77.66      & \textbf{81.82} & 86.42      & 77.03      \\
                        SSD300 w/   GAF  & Arctan $(0.2, 10)$ & \textbf{78.21} & \textbf{84.07} & 84.65      & \textbf{78.01} & \textbf{71.55} & \textbf{54.57} & 85.30      & 87.03      & 87.54      & 62.84      & \textbf{82.86} & 74.70      & 84.77      & 86.65      & \textbf{85.67} & \textbf{80.61} & 52.20      & \textbf{78.03} & 78.84      & \textbf{86.79} & \textbf{77.57} \\
                        \bottomrule
                    \end{tabular}
        }
    }
    \label{TBOD}
\end{table*}
The comparison experiments are mainly conducted to investigate the effect of embedding the GAF in the performance of neural network models. ResNet-$18$, ResNet-$34$, ResNet-$50$, SE-ResNet-$18$, SE-ResNet-$34$, and SE-ResNet-$50$ are evaluated on ImageNet \cite{ResNet, hu2018squeeze}. The comparisons on ImageNet between the original method, two types of gradient clipping methods, and the GAF are conducted, with the corresponding results presented in Table \ref{TBImageNetCom}. This table shows that the effect of the gradient value clipping method on performance is unstable, while the gradient norm clipping method brings degradation. By contrast, the GAF significantly improves the performance of all models shown in Table \ref{TBImageNetCom}. DenseNet-CIFAR, DenseNet-$169$, DPN-$26$, GoogLeNet, ResNet-$50$, and RegNetX-$200$MF are evaluated on CIFAR-$100$ \cite{huang2017densely, chen2017dual, szegedy2015going, radosavovic2020designing, ResNet}. Note that DenseNet-CIFAR means DenseNet ($k=12$, depth$=121$). Two sets of parameter schemes are evaluated for the GAF, the first one is $\{ \alpha = 0.1, \beta = 20\}$ (abbreviated as $(0.1, 20)$) and the other is arctan $\{ \alpha = 0.2, \beta = 10 \}$ (abbreviated as $(0.2, 10)$). In order to evaluate the performance of different types of GAFs, comparison results of SGDM on CIFAR-$100$ are presented in Table \ref{TBGAFs}, from which it can be summarized that all kinds of GAFs improve the performance of the involved models. In addition to SGDM, Adam optimizer is also embedded with the GAF and evaluated on both CIFAR-$10$ and CIFAR-$100$, and improvements are also observed, as shown in Table \ref{TBAdam}. On PASCAL VOC, SSD300 \cite{SSD2016} with and without the GAF are used for comparison. Table \ref{TBOD} shows that after using the GAF, the mAP is improved. Overall, the GAF consistently and effectively improves the performance of involved models.

\section{Conclusion}\label{sec.5}

In this paper, the GAF has been proposed for ameliorating the ill-conditioned problem, the vanishing gradient problem, the exploding gradient problem, and the saddle point problem in one shot. Theoretical analyses have demonstrated the feasibility and effectiveness of the GAF. Moreover, the implementation of which requires roughly one line of code. Comparative experiments for the GAF have been conducted by exploiting some classical or state-of-the-art convolutional neural networks on ImageNet, CIFAR, and PASCAL VOC. The experiments have shown that the trained models' performance equipped with the proposed GAF is significantly better than the original ones.

The proposed GAF also shows extensive and promising future research directions. Additional theoretical properties of the GAF can be further explored, such as how the GAF influences the generalization. Also, some adaptive adjustment methods for automatically determine the hyperparameters are worthy of investigation. On the other hand, from the perspective of applying the GAF, other applications such as some natural language processing or reinforcement learning tasks could be taken into consideration.


\end{document}